%% file: paper.tex
\documentclass[]{mosi}

\usepackage{helvet}
\usepackage{times}
\usepackage{latexsym}

\usepackage{amsmath}
\usepackage{amssymb}
\usepackage{amsfonts}

\usepackage{amsthm}
\usepackage{natbib}

\usepackage[T1]{fontenc}
\usepackage[utf8]{inputenc}

\usepackage{microtype}
\usepackage{inconsolata}
\usepackage{nicefrac}

\usepackage{graphicx}
\usepackage{subcaption}
\usepackage[export]{adjustbox}
\usepackage{xcolor}
\usepackage[table]{xcolor}
\usepackage{colortbl}
\usepackage{booktabs}
\usepackage{multirow}
\usepackage{makecell}
\usepackage{array}
\usepackage{tabularx}
\usepackage{threeparttable}
\usepackage{siunitx}
\usepackage{float}
\usepackage{wrapfig}
\usepackage{subcaption}

\usepackage[toc,page,header]{appendix}
\usepackage{hyperref}
\usepackage{url}
\usepackage{titletoc}
\usepackage{minitoc}
\usepackage{ragged2e}
\usepackage{caption}
\usepackage{enumitem}
\usepackage[hang,flushmargin]{footmisc}
\usepackage{setspace}
\usepackage{etoolbox}
\usepackage{comment}

\usepackage{pgfplots}
\usepackage{tikz}
\usepackage{pgf-pie}

\usepackage{tcolorbox}
\tcbuselibrary{breakable,skins}
\usepackage{fvextra}
\usepackage{fancyvrb}
\usepackage{listings}
\usepackage{mdframed}

\usepackage{pifont}

\definecolor{lightblue}{RGB}{200, 230, 255}
\definecolor{headerblue}{RGB}{150, 200, 255}
\definecolor{oursgray}{gray}{0.95}

\definecolor{MossCyan}{HTML}{82D9FF}
\definecolor{MossBlue}{HTML}{82B1FF}

\definecolor{tickG}{HTML}{00C853}
\definecolor{crossR}{HTML}{FF1744}

\newtheorem{theorem}{Theorem}

\newcommand{\faGithub}{\raisebox{-0.2ex}{\includegraphics[height=2.0ex]{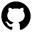}}}

\newcommand{\methodname}{\textsc{MEDS}}

\newcommand{\best}[1]{\textbf{#1}}
\newcommand{\second}[1]{\underline{#1}}

\newtcolorbox{promptbox}[2][]{
    colback=white,
    coltext=black,
    arc=3mm,
    boxrule=0.5pt,
    colframe=black!60!white,
    title={#2},
    colbacktitle=black,
    coltitle=white,
    fonttitle=\bfseries,
    top=8pt,
    bottom=8pt,
    left=10pt,
    right=10pt,
    breakable,
    before upper={%
        \linespread{1}\selectfont
        \setlength{\parskip}{1ex plus 0.2ex minus 0.2ex}%
        \setlength{\parindent}{0pt}%
    },
    #1
}

\title{The Past Is Not Past: Memory-Enhanced Dynamic Reward Shaping}

\author{
Yang Liu$^{1,2,*}$,
Enxi Wang$^{1,*}$,
Yufei Gao$^{1,*}$,
Weixin Zhang$^{1,*}$,
Bo Wang$^{1}$,
Zhiyuan Zeng$^{1}$,
Yikai Zhang$^{1}$,
Yining Zheng$^{1,2,\dagger}$,
Xipeng Qiu$^{1,2,\dagger}$\\[2mm]
{\normalfont \normalsize $^{1}$Fudan University}\\
{\normalfont \normalsize $^{2}$Shanghai Innovation Institute}
}

\abstract{
Despite the success of reinforcement learning for large language models, a common failure mode is reduced sampling diversity, where the policy repeatedly generates similar erroneous behaviors. Classical entropy regularization encourages randomness under the current policy, but does not explicitly discourage recurrent failure patterns across rollouts. We propose \textbf{\methodname}, a \textbf{M}emory-\textbf{E}nhanced \textbf{D}ynamic reward \textbf{S}haping framework that incorporates historical behavioral signals into reward design.
By storing and leveraging intermediate model representations, we capture features of past rollouts and use density-based clustering to identify frequently recurring error patterns. Rollouts assigned to more prevalent error clusters are penalized more heavily, encouraging broader exploration while reducing repeated mistakes. Across five datasets and three base models, \methodname~consistently improves average performance over existing baselines, achieving gains of up to 4.13 pass@1 points and 4.37 pass@128 points. Additional analyses using both LLM-based annotations and quantitative diversity metrics show that \methodname~increases behavioral diversity during sampling.
}

\checkdata[\raisebox{-0.1ex}{\faGithub}\hspace{0.4em}Github]{\url{https://github.com/Linxi000/MEDS}}

\begin{document}
\maketitle
\begingroup
\renewcommand{\thefootnote}{\fnsymbol{footnote}}
\setcounter{footnote}{0}
\footnotetext[1]{These authors contributed equally.}
\footnotetext[2]{Corresponding authors.}
\endgroup

\input{sections/introduction}
\input{sections/related_work}

\input{sections/method}
\input{sections/experiments}

\input{sections/conclusion}
\input{sections/limitations}
\bibliographystyle{unsrtnat}
\bibliography{main}
\clearpage
\appendix
\input{sections/appendix}

\end{document}

%% file: sections/introduction.tex
\section{Introduction}

\begin{figure}[t]
\centering

\begin{subfigure}[t]{0.482\linewidth}
    \centering
    \includegraphics[width=\linewidth]{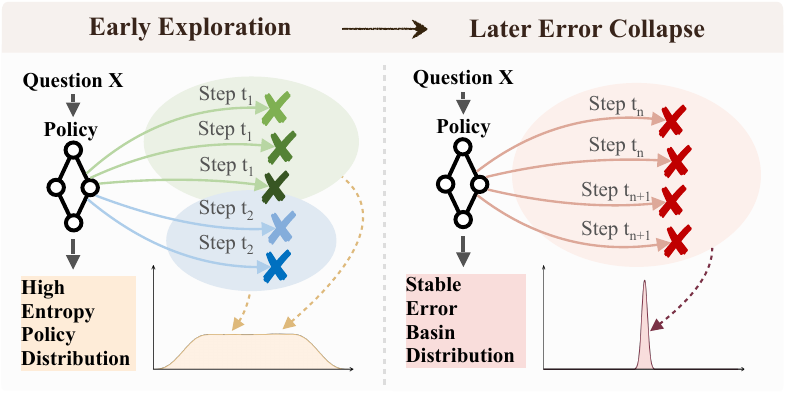}
    \caption{}
    \label{fig:problem-a}
\end{subfigure}
\hfill
\begin{subfigure}[t]{0.247\linewidth}
    \centering
    \includegraphics[width=\linewidth]{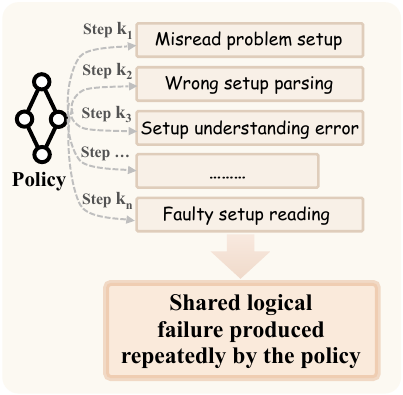}
    \caption{}
    \label{fig:problem-b}
\end{subfigure}
\hfill
\begin{subfigure}[t]{0.257\linewidth}
    \centering
    \includegraphics[width=\linewidth]{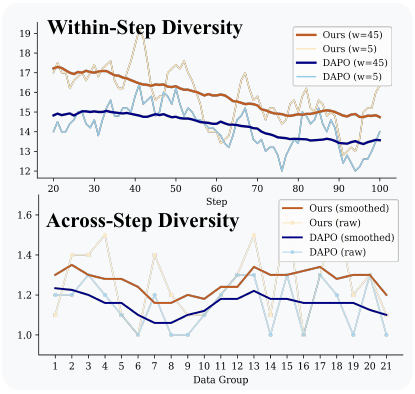}
    \caption{}
    \label{fig:problem-c}
\end{subfigure}

\caption{Illustration of error collapse in on-policy optimization and its effect on output diversity. \textbf{(a)} During on-policy optimization, the policy evolves from a high-entropy distribution into a narrow, stable error basin.
\textbf{(b)} Different errors share the same failure mode and become progressively reinforced during training.
\textbf{(c)} Comparison of output diversity between DAPO and our method, with detailed description provided in Section~\ref{exp_exploration}.}

\label{fig:problem}
\end{figure}

With the advancement of the fundamental capabilities of Large Language Models (LLMs), reinforcement learning has achieved significant success across various domains~\citep{grpo, StepCoder, ExecutionCodeGeneration}. By incorporating reward signals, whether derived from rule-based evaluation~\citep{grpo} or proxy models~\citep{InstructGPT}, LLMs iteratively alternate between a sampling phase and a gradient-based optimization phase. As the model’s performance is optimized toward maximizing the expected return, designing the reward scoring structure becomes the primary way to guide the model's behavior~\citep{allroads2likelihood}.

A well-known failure mode of such on-policy optimization is that, as training progresses, the policy often collapse into a narrow and stereotyped set of behaviors~\citep{DAPO,ips,EEPO}. This degeneration produces highly repetitive responses that not only weaken effective exploration and waste on-policy samples, but also entrench the model in self-reinforcing erroneous reasoning trajectories\citep{negative_cdq}. A series of works~\citep{HaarnojaZAL18,wan2026dsdr,entropy-2016} introduce entropy regularization to promote output diversity by injecting randomness at the level of the current policy distribution. However, in LLM reasoning, such distribution-level stochastic exploration is often insufficient due to the vast action space. Moreover, as illustrated in Figure~\ref{fig:problem}, once training drifts into a particular error pattern, the policy repeatedly reproduce the same mistake by sampling slightly different reasoning expressions, remaining trapped in this pattern despite continued optimization over many steps.

Inspired by the dynamic reward mechanisms observed in human learning \citep{neural_error}, where individuals impose stronger psychological penalties on recurring mistakes to reinforce avoidance, we propose a \textbf{M}emory-\textbf{E}nhanced \textbf{D}ynamic reward \textbf{S}haping framework (\methodname). Unlike static reward functions, our method dynamically records historical error patterns and imposes incremental penalties on repetitive failure paths. This mechanism not only facilitates the correction of persistent errors but also encourages the model to escape local optima, thereby enhancing exploration in complex solution spaces. We further provide theoretical analysis showing that additional penalties on repeated errors can improve performance.

In practice, by reusing the layer-wise logits already produced during the forward pass, we construct feature vectors of the model's implicit reasoning trajectory. This reuse introduces almost no additional computational overhead and serves as a lightweight yet effective representation. We then maintain an error memory that is updated online throughout training. By incorporating HDBSCAN-based clustering~\citep{hdbscan}, this memory adaptively captures error types and their density structure, enabling the training objective to adjust penalty strength according to the recurrence of failure patterns. 

Empirical results across multiple benchmarks show that our method improves both pass@1 and pass@128 across all three models, with gains of up to 4.13 points in pass@1 and 4.37 points in pass@128. The largest pass@128 improvement is observed on Qwen3-8B, where performance increases from 70.81 to 82.67 on a single dataset, corresponding to a relative gain of 17\%. We also analyze changes in the model's exploration behavior, demonstrating that our method increases the diversity of exploration. In addition, we calculate the consistency between logit-based clustering and LLM-based clustering, showing that logits can serve as a proxy for reasoning structure. The positive correlation between clustering quality and downstream model performance further validates the effectiveness of our approach. Our main contributions are summarized as follows:





\begin{enumerate}
\item We propose \methodname, a novel memory-enhanced dynamic reward shaping framework for LLM reinforcement learning. To the best of our knowledge, this is the first work to explicitly incorporate historical error patterns into reward modeling, enabling the policy to recognize and avoid recurrent failure behaviors rather than merely increasing randomness under the current policy. This design provides more informative supervision during training and effectively mitigates repeated errors.

\item We provide theoretical analysis supporting the effectiveness of memory-based dynamic penalties in improving learning and mitigating recurrent failure modes.

\item We introduce a representation-level criterion for response similarity that leverages layer-wise logits as a compact and efficient proxy, with little additional computational overhead.

\item We conduct extensive experiments across multiple datasets and models, showing that \methodname~consistently achieves the best overall performance and substantially improves both pass@1 and pass@128. Further analyses show that our method increases exploration diversity, that logit-based clustering aligns well with LLM-based semantic clustering, and that better recognition of repeated error patterns is strongly correlated with downstream performance gains.
\end{enumerate}

%% file: sections/related_work.tex
\section{Related Work}
\subsection{Reinforcement Learning with Verifiable Rewards}
The application of reinforcement learning in LLMs relies heavily on the design of reward signals. And Reinforcement Learning with Verifiable Rewards (RLVR) has become the dominant paradigm due to the simplicity of its reward signals.
Representative works adopt diverse verifiable criteria across tasks. \citet{mathvarify} and \citet{grpo} assign binary rewards based on exact answer matching. \citet{zhu2025codevr1} automatically verifies correctness by executing generated programs in a sandbox. \citet{ThinkingWithVideos} uses both temporal localization quality and multiple-choice Video QA accuracy as reward signals. \citet{chen2025learning} employs a multimodal large language model (MLLM) to evaluate similarity between generated and target images. \citet{RERANK} adopts metrics such as nDCG@10 as the primary reward.
In all these works, the verifier operates independently within each epoch and assigns rewards based solely on current samples, overlooking useful information in historical trajectories. Our approach instead leverages trajectories from earlier stages to guide later sampling.

\subsection{Reward Shaping for Diversity Exploration}
\label{sec:relatedwork-rewardshaping}
During RLVR training, policy updates tend to over-optimize dominant solution patterns, leading to collapse of reasoning diversity~\citep{LightmanKBEBLLS24,outcome-based-feedback}. Traditional approaches mitigate this by constraining policy shift via KL regularization~\citep{InstructGPT,ppo} or encouraging randomness via entropy regularization~\citep{HaarnojaZAL18}.

Beyond implicit regularization, several works explicitly reward rare reasoning behaviors. \citet{Rewarding-the-Rare} clusters rollouts by solution strategy and inversely weights advantages by cluster size. \citet{wan2026dsdr} promotes diversity among correct trajectories with token-level entropy regularization. However, these methods operate only on current-batch differences and do not track solution patterns over long-term training history.


To leverage historical information, \citet{Outcome-based-Exploration} introduces UCB-style exploration rewards based on historical outcome frequencies. Yet it still lacks fine-grained cross-step reasoning behavior tracking. Another line of work~\citep{Auto-MC-Reward,Multi-Objective-Evolution,LatentReward,REvolve,Eureka} uses external models to adaptively adjust rewards based on historical performance, but the additional computational cost limits scalability. In contrast, our method captures reasoning patterns via layer-wise logits, enabling history-aware reward adjustment with minimal overhead.



\subsection{Inner Thoughts of LLMs}
Works that attempt to understand and utilize how models reason and arrive at their final conclusions are an important source of inspiration for our approach. \citet{modeledit} localizes factual knowledge in the model by examining logit changes. 
Subsequent studies~\citep{anthropicgoldengate, anthropicinductionhead, lorsa, llamascope} use Sparse Autoencoders (SAEs) and related logit-intervention techniques to obtain a more fine-grained understanding of LLM reasoning processes.
These efforts have further motivated the use of logits for analysis and supervision. For example, \citet{VerifyCOT} detects erroneous reasoning by verifying the inner attribute graph, while \citet{BPO} leverages layer-wise logits to provide more fine-grained supervision.

A key component of our work is to identify shared patterns across different reasoning trajectories. Inspired by prior work on logit-based model analysis, we cluster logits to capture such patterns without requiring extra models. We describe this approach in detail in Section~\ref{method:logits_cluster}.

%% file: sections/method.tex
\section{Method}
\subsection{Preliminaries}
Given an input  $\mathbf{x} \sim D$, the LLM is modeled as a policy $p_\theta$ that generates a response sequence  $ \mathbf{y} \sim p_\theta(\mathbf{y}\mid \mathbf{x})$, where  \( \mathbf{y} = \{y_1, y_2, \ldots, y_L\} \in \mathcal{Y} \) represents the complete response and \( y_t \) denotes the token generated at step \(t\). The response is evaluated and mapped to a scalar reward via a reward function 
\( r : \mathcal{X} \times \mathcal{Y} \rightarrow \mathbb{R} \). In practical optimization, a KL-regularized objective is adopted, and we maximize the expected return within a trust region. The corresponding training objective is
\begin{equation*}
\begin{aligned}
\max_{\theta}\ \mathbb{E}_{\mathbf{x}\sim D}\Big[
&\mathbb{E}_{\mathbf{y} \sim p_\theta(\cdot \mid \mathbf{x})} \big[r(\mathbf{x}, \mathbf{y})\big] - \frac{1}{\eta}\,\mathrm{KL}\!\big(p_\theta(\cdot \mid \mathbf{x})\,\|\,p_{\mathrm{ref}}(\cdot \mid \mathbf{x})\big)
\Big]
\end{aligned}
\end{equation*}
where \( p_{\mathrm{ref}} \) is the reference policy that constrains the update space, and \( \eta \) is the regularization coefficient controlling the strength of the constraint. 

\begin{figure*}[t]
\centering
\includegraphics[width=\textwidth]{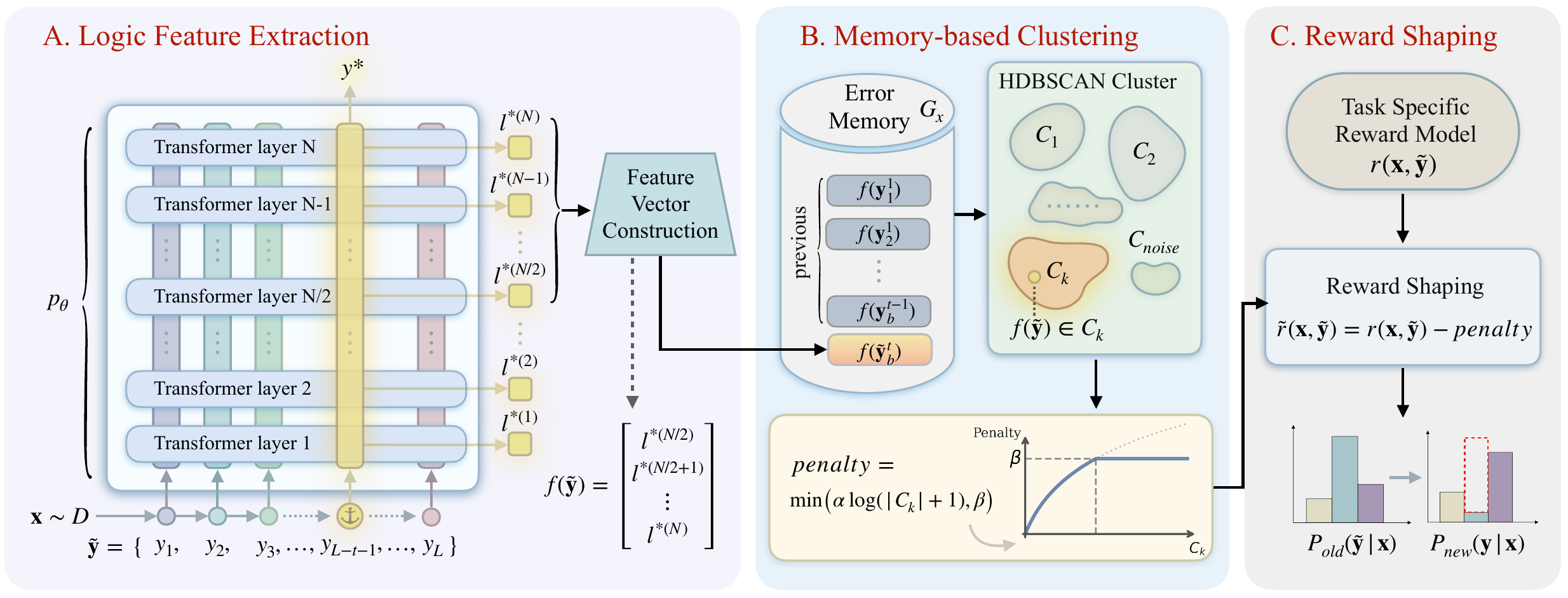}  
\caption{\textbf{Overview of \methodname.} The process consists of three stages: \textbf{(A) Logic Feature Extraction:} For an input \(x\), we sample a response \(\tilde{y}\) and extract a logic feature vector \(f(\tilde{y})\) from the layer-wise logits of the first token \(y^*\) in the final answer, using the latter half of the Transformer layers to capture reasoning patterns. \textbf{(B) Memory-based Clustering:} We maintain a per-prompt error memory \(G_x\) and apply HDBSCAN to group historical responses with similar logic features. \textbf{(C) Reward Shaping:} For a response assigned to cluster \(C_k\), we compute a penalty \(\min(\alpha \log(|C_k|+1), \beta)\) proportional to the cluster size, and adjust the reward as \(\tilde{r}(x,\tilde{y}) = r(x,\tilde{y}) - \text{penalty}\), which is then used to update the policy.}

\label{fig:overview}
\vspace{-2pt}
\end{figure*}

\subsection{Theoretical Benefits of Penalizing Repeated Errors}

With the mathematical formulation above, we theoretically analyze the benefits of penalizing repeated errors.

\begin{theorem}
Suppose there are two different reward signals, $\mu_1 = r(\mathbf{x}, \mathbf{y})$ and $\mu_2 = r(\mathbf{x}, \mathbf{y}) - \lambda c(\mathbf{y})$, and let $q_1$ and $q_2$ denote the corresponding updated policies under these reward signals. Let $J(q) = \mathbb{E}_{\mathbf{x}, \mathbf{y} \sim q}[r(\mathbf{x}, \mathbf{y})]$. Then it holds that
\begin{align*}
J(q_2) \geq J(q_1),
\end{align*}
where $\lambda$ is the repetition penalty coefficient. $c(\mathbf{y}): X \times Y \rightarrow \mathbb{N}$ denotes the indicator of the number of repeated errors, where a value greater than zero indicates that $\mathbf{y}$ corresponds to an error that has been repeatedly sampled in the past. The function $r(\mathbf{x}, \mathbf{y})$ is a task-specific scoring function. The function \(c(\mathbf{y})\) increases monotonically with the number of occurrences.
\end{theorem}
We provide a detailed proof in Appendix~\ref{subsec:theoratical_proof}. Intuitively, repeatedly sampled errors indicate higher probability mass under the current policy, and adding an extra penalty helps reallocate this probability. Since the penalty is applied to repeated errors, the expected return is non-decreasing from the perspective of optimization. \textbf{The indicator function acts as a memory for the reward model and adjusts the score based on the historical sampling behavior of the policy.}

\subsection{Logic Feature Extraction}
\label{method:logits_cluster}
When implementing the indicator function \(c(\mathbf{y})\), the key aspect is to handle the relationships among different responses. To do so efficiently, we directly reuse the layer-wise logits of each response.

Specifically, suppose the final answer tokens are \(\{y_{L-t}, \ldots, y_L\} \subset \{y_1, \ldots, y_L\}\), and let \(y^* = y_{L-t}\) denote the first token of the final answer. The corresponding logits at each layer $n$ are \(l^{(n)}_{y_{L-t}} \in \mathbb{R}^{V}\), where V is the size of vocabulary. We select the logit corresponding to this token position, denoted as \(l^{*(n)} \in \mathbb{R}\), and aggregate them across layers. The evolution of logits for the final answer depends on feature activations across different layers \citep{anthropicgoldengate}, and therefore contains information about the model’s internal reasoning process. Figure~\ref{fig: circuit} provides an intuitive illustration of this process. We also present further examples in Section~\ref{subsec:logits_pattern}, showing that this correlation construction captures more fundamental similarities than approaches based solely on surface-level string matching. 
\begin{figure*}
    \centering
    \includegraphics[width=\linewidth]{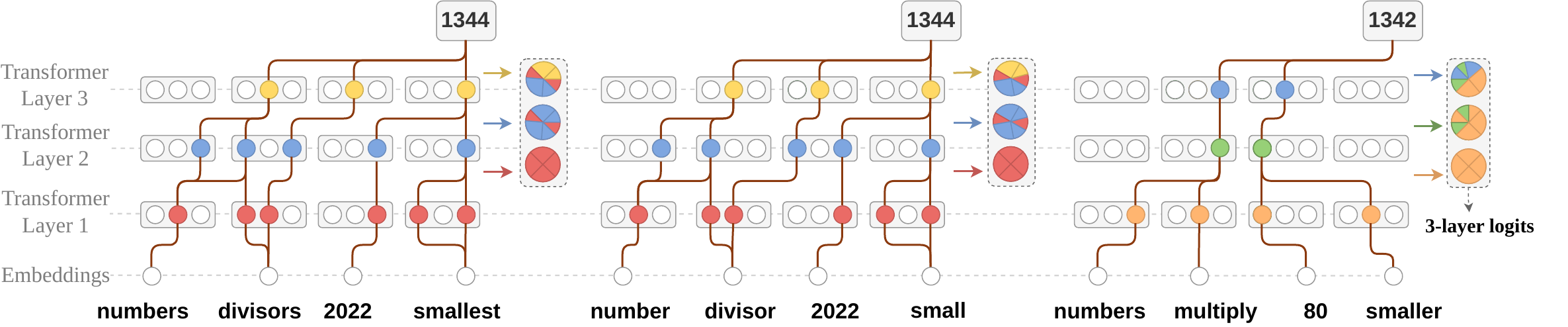}
    \caption{The logits aggregation process, exemplified by a 3-layer model feature graph. Each layer contributes a scalar logit at the final answer token position, whose value is determined by the activated computational circuits. The \textbf{left} and \textbf{middle} graphs activate nearly identical circuits, resulting in nearly identical 3-layer logit vectors. In contrast, the \textbf{right} graph exhibits a substantially different circuit structure, yielding divergent logits.}
    \label{fig: circuit}
    \vspace{-2pt}
\end{figure*}
The logits vector can be processed in several ways. The two most straightforward approaches are direct aggregation and aggregation after differencing. For clarity, we adopt the direct aggregation approach in our description here and compare the performance of several alternatives in Section~\ref{subsec:cluster_quality}. Since earlier layers typically model simpler semantic information \citep{Engram}, we use the logits from the latter half of the layers. The resulting feature representation is
\[
f(\mathbf{y}) = \text{concat}\left(l^{*(n)} \mid n=\frac{N}{2}, \ldots, N \right) \in \mathbb{R}^{\frac{N}{2}}.
\]

\subsection{Cluster-based Reward Shaping}
Based on the constructed response representations, we can directly compute the indicator function and perform the corresponding reward shaping. For each specific prompt $\mathbf{x}$, we maintain a set of feature representations for all responses sampled across historical rollouts:
\[
G_{\mathbf{x}} = \{f(\mathbf{y}_1^1), f(\mathbf{y}_2^1), \ldots, f(\mathbf{y}_b^t)\},
\]
where \(b\) denotes the number of responses sampled for each prompt and \(t\) denotes the training epoch. We then apply HDBSCAN clustering on this set to obtain
\[
G_\mathbf{x} = \{C_1, C_2, \ldots, C_K, C_{\text{noise}}\},
\]
where \(C_i\) denotes the \(i\)-th cluster, and the number of clusters \(K\) is determined dynamically.

For a specific response \(\tilde{\mathbf{y}}\), let \(C_k\) denote the cluster to which $ f(\tilde{\mathbf{y}})$ is assigned, where \(|C_k|\) denotes the size of the cluster. The corresponding indicator is defined as \(c(\tilde{\mathbf{y}}) = \log(|C_k| + 1)\), a monotonically increasing transformation of cluster size that preserves the ordering used in the theoretical analysis in Appendix B. The reward is then adjusted as follows:
\[
\begin{aligned}
\tilde{r}(\mathbf{x}, \tilde{\mathbf{y}})
&= r(\mathbf{x}, \tilde{\mathbf{y}}) - \min\!\left(\alpha \log(|C_k| + 1), \beta\right), \text{s.t. } f(\tilde{\mathbf{y}}) \in C_k .
\end{aligned}
\]
where \(\alpha\) and \(\beta\) are hyperparameters controlling the penalty strength and the upper bound of the penalty.

%% file: sections/experiments.tex
\section{Experiments}
\input{tables/main_results}

\subsection{Experimental Setup}

\paragraph{Training Details}
We train on three base models: Qwen3-1.7B, Qwen3-8B, and Qwen2.5-Math-7B. Qwen3-8B and Qwen2.5-Math-7B represent models with and without explicit reasoning processes, while Qwen3-1.7B serves as a model with a different scale. The training corpus is constructed by merging the DAPO-Math-17K dataset~\citep{DAPO} with the levels 3-5
of the MATH subdataset~\citep{HendrycksBKABTS21,R1-Zero-Like}. All problems are formatted using the Qwen-Math template (Appendix~\ref{subsec:qwen-math-template}). All reinforcement learning training is implemented based on the \textit{verl} framework~\citep{ShengZYWZZPL025}.

Building on DAPO, our work further introduces a clustering penalty term: for each prompt, we maintain a memory of the logits representations of incorrect samples, where we concatenate the logits from the last 14 Transformer layers and apply L2 normalization, then perform clustering in the logit space using HDBSCAN. During clustering, we set the minimum cluster size to 2, the minimum number of samples to 1, and use Euclidean distance as the distance metric. As for the penalty coefficients, we set $\alpha = 0.1$ and $\beta = 0.2$ for Qwen3-1.7B, Qwen2.5-Math-7B, and set $\alpha = 0.02$ and $\beta = 0.04$ for Qwen3-8B.

In all training runs, we use a prompt-level batch size of 512 and generate 16 rollouts per prompt; the maximum prompt length is 1024 tokens. The maximum response length is 7168 tokens for Qwen3-1.7B, Qwen3-8B, and 3072 tokens for Qwen2.5-Math-7B. All training is conducted on H200 hardware. Other hyperparameters can be found in Appendix~\ref{subsec:training-hyperparams}.

\paragraph{Baselines}
We compare the proposed \methodname~with two commonly used RL algorithms: GRPO and DAPO. In addition, as discussed in section~\ref{sec:relatedwork-rewardshaping}, we also compare with a approach that encourages exploration through entropy regularization~\citep{Cheng2025ReasoningWE}. The hyperparameters and training objectives configurations for these algorithms are provided in appendix~\ref{subsec:training-hyperparams}. For reference, we also report the performance of the base model without fine-tuning as a baseline.

\paragraph{Evaluation}
We evaluate our method on five mathematical reasoning benchmarks covering different difficulty levels and problem types: AIME24~\citep{li2024numinamath}, AMC23~\citep{li2024numinamath}, MATH500~\citep{HendrycksBKABTS21}, Minerva~\citep{LewkowyczADDMRS22}, and OlympiadBench~\citep{HeLBHTSHHHZLQL024}.

For pass@$k$ evaluation, we set the number of samples to $n=128$ and report both pass@$1$ and pass@$128$. During evaluation, we also use the Qwen-Math template. The temperature is set to $1.0$ and top-$p$ to $1.0$ for all models. The random seed is fixed to $0$, and the maximum number of tokens is kept consistent with the training configuration.

\subsection{Main Results}
\begin{figure}[t]
    \centering
    \includegraphics[width=\linewidth]{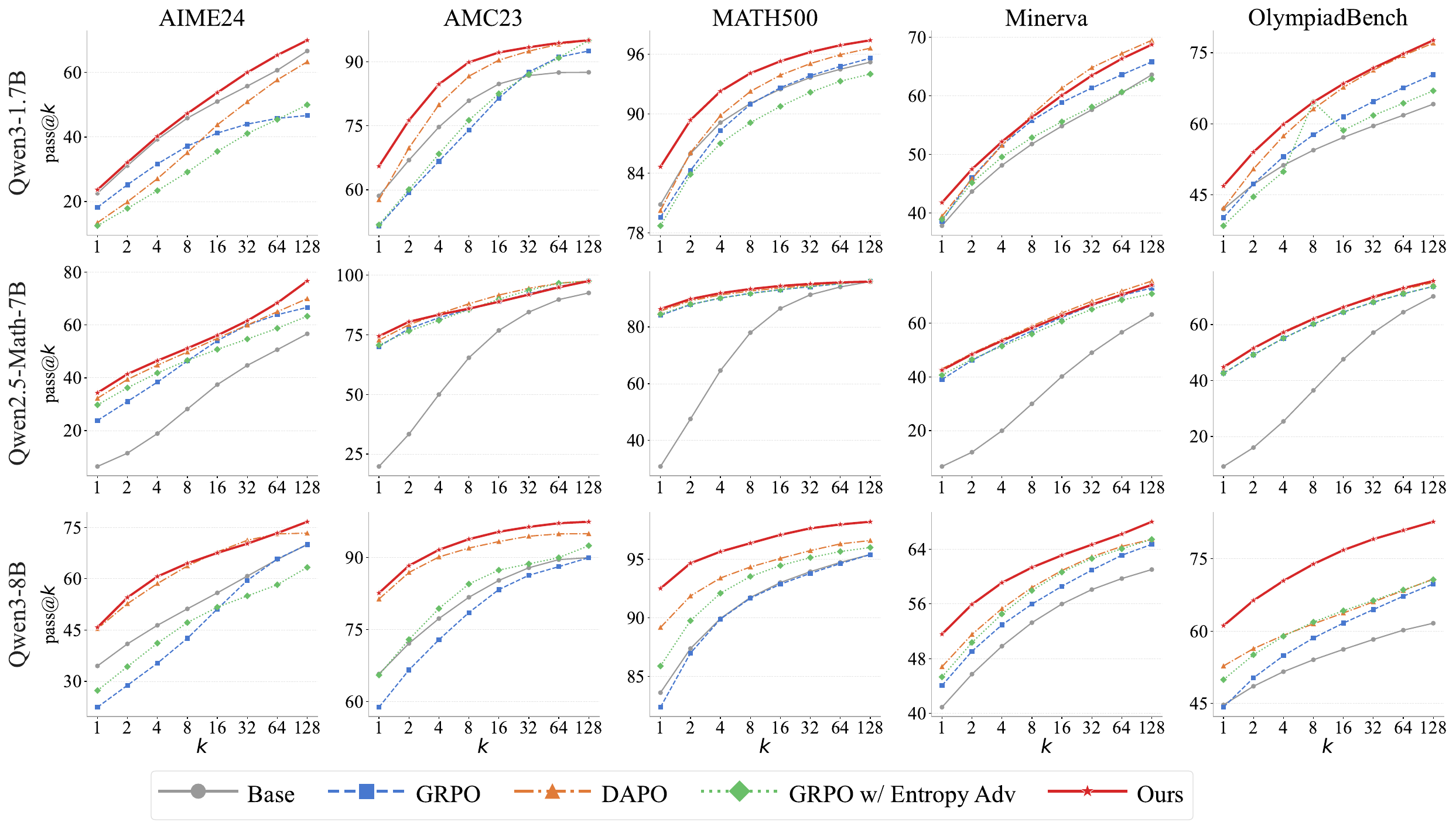}
    \caption{Pass@$k$ across $k \in \{1, 2, 4, 8, 16, 32, 64, 128\}$ on five math benchmarks for three model scales. Each row corresponds to one model; each column to one dataset. Ours consistently matches or outperforms Base, GRPO, DAPO, and GRPO with entropy advantage across settings.}
    \label{fig:passk_full}
\end{figure}

Table~\ref{tab:main_results} summarizes the performance across three base models and five mathematical reasoning benchmarks. Our method consistently achieves the best overall results.


Across all three model scales, our method consistently achieves the best overall performance, outperforming both the base model and strong baselines such as DAPO. This advantage is observed on both Qwen3-1.7B and Qwen2.5-Math-7B, and becomes even more pronounced on Qwen3-8B model, suggesting that the benefits of our method scales with the capability of the base model.
We also plot the pass@\(k\) curves, where \(k\) takes values in \(\{1, 2, 4, 8, 16, 32, 64, 128\}\). The corresponding results are shown in Figure~\ref{fig:passk_full}, where our method achieves the best performance in most cases.

Training efficiency remains close to DAPO: on Qwen2.5-Math-7B, \methodname{} takes 8.46 minutes for 50 steps and 9.73 minutes for 100 steps, compared with 8.00 and 8.95 minutes for DAPO.
Our method introduces only a small computational overhead compared with DAPO. Considering the consistent performance improvements, this overhead is relatively modest.
\subsection{Impact on Exploration Behavior}
\label{exp_exploration}
\begin{figure}[t]
\centering
\begin{minipage}[t]{0.37\columnwidth}
    \vspace{0pt}
    \centering
    \includegraphics[width=\linewidth]{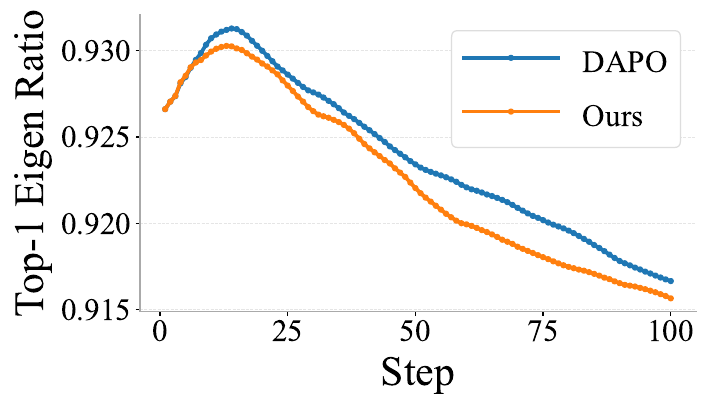}
\end{minipage}
\hspace{0.04\columnwidth}
\begin{minipage}[t]{0.27\columnwidth}
    \vspace{7pt}
    \centering
    \renewcommand{\arraystretch}{1.3}
    \resizebox{\linewidth}{!}{%
    \begin{tabular}{lc}
    \hline
    Logit Aggregation & Correlation(\%) \\ \hline
    single cluster & 45.16 \\ \hline
    28--diff & 52.52 \\
    14--diff & 54.71 \\ \hline
    28 & 53.87 \\
    14 & \textbf{61.2} \\ \hline
    \end{tabular}%
    }
\end{minipage}

\caption{\textbf{(Left)} Top-1 eigen ratio over training steps for DAPO and Ours. Our method consistently yields lower ratios at later steps, indicating less concentration on a single dominant error direction. \textbf{(Right)} Correlation of different logit aggregation strategies on Qwen2.5-Math-7B. Single cluster assigns all logits to one cluster; 28 and 14 use direct aggregation over all 28 layers or the last 14 layers; 28--diff and 14--diff use aggregation over layer-wise differences between adjacent logits.}

\label{fig:eigen_ratio_claude_correlation}
\end{figure}

To better understand how our method influences the model's exploration behavior during reasoning, we conduct a detailed analysis from both behavioral and representational perspectives. Specifically, we first employ an LLM-based evaluation to measure \textbf{Within-Step Diversity} and \textbf{Across-Step Diversity}, which capture two complementary aspects of rollout diversity. We then further analyze the statistical structure of the logits to provide additional evidence of diversity at the representation level.

Within-Step Diversity measures the diversity of rollouts generated by DAPO and \methodname~at the same training step for the same input prompt. Since the first training epoch is mainly used for memory formation and does not involve reward shaping, we restrict this analysis to checkpoints from the second epoch, spanning approximately step 20 to step 100. At each step, we randomly sample 10 groups of 160 rollouts from each method, resulting in 1,600 rollout groups in total, and use \textit{Claude-Haiku-4.5} with the prompt in Appendix~\ref{app:prompt} to score group-level output diversity.

Across-Step Diversity instead evaluates whether later-stage rollouts exhibit new reasoning patterns relative to earlier ones for the same model and input prompt. For this purpose, we construct 21 pairs of data groups with varying step intervals, each roughly one epoch apart, where each data group contains 10 rollout groups from DAPO and 10 from \methodname. Using the same prompt-based protocol, the LLM compares rollouts from the earlier and later steps and scores the extent to which the later outputs introduce novel reasoning patterns. In both settings, higher scores indicate greater diversity.

Beyond these LLM-based evaluations, we further analyze the covariance structure of the stored logits across all steps and compute the ratio between the largest eigenvalue and the sum of all eigenvalues. We refer to this metric as the \textit{Top-1 Eigen Ratio}. Its formal definition is given in Appendix~\ref{app:top1_eigen_ratio}. A smaller value of this ratio indicates a more uniform spread of the logits across different directions in the representation space.

\paragraph{Our method improves exploration across multiple metrics.} The LLM-based evaluation results are shown in the bottom-right panel of Figure~\ref{fig:problem}, and the \textit{Top-1 Eigen Ratio} results are presented in Figure~\ref{fig:eigen_ratio_claude_correlation} (left). We observe that rollout diversity generally declines as training progresses. However, our method consistently attains higher exploration scores than the baseline under both Within-Step Diversity and Across-Step Diversity. This trend is further supported by the \textit{Top-1 Eigen Ratio}, where our method yields consistently lower values than the baseline, indicating a more diverse exploration behavior.

\subsection{Logits can Reflect Reasoning Patterns}
\label{subsec:logits_pattern}
\begin{figure}[t]
\centering
\includegraphics[width=\linewidth]{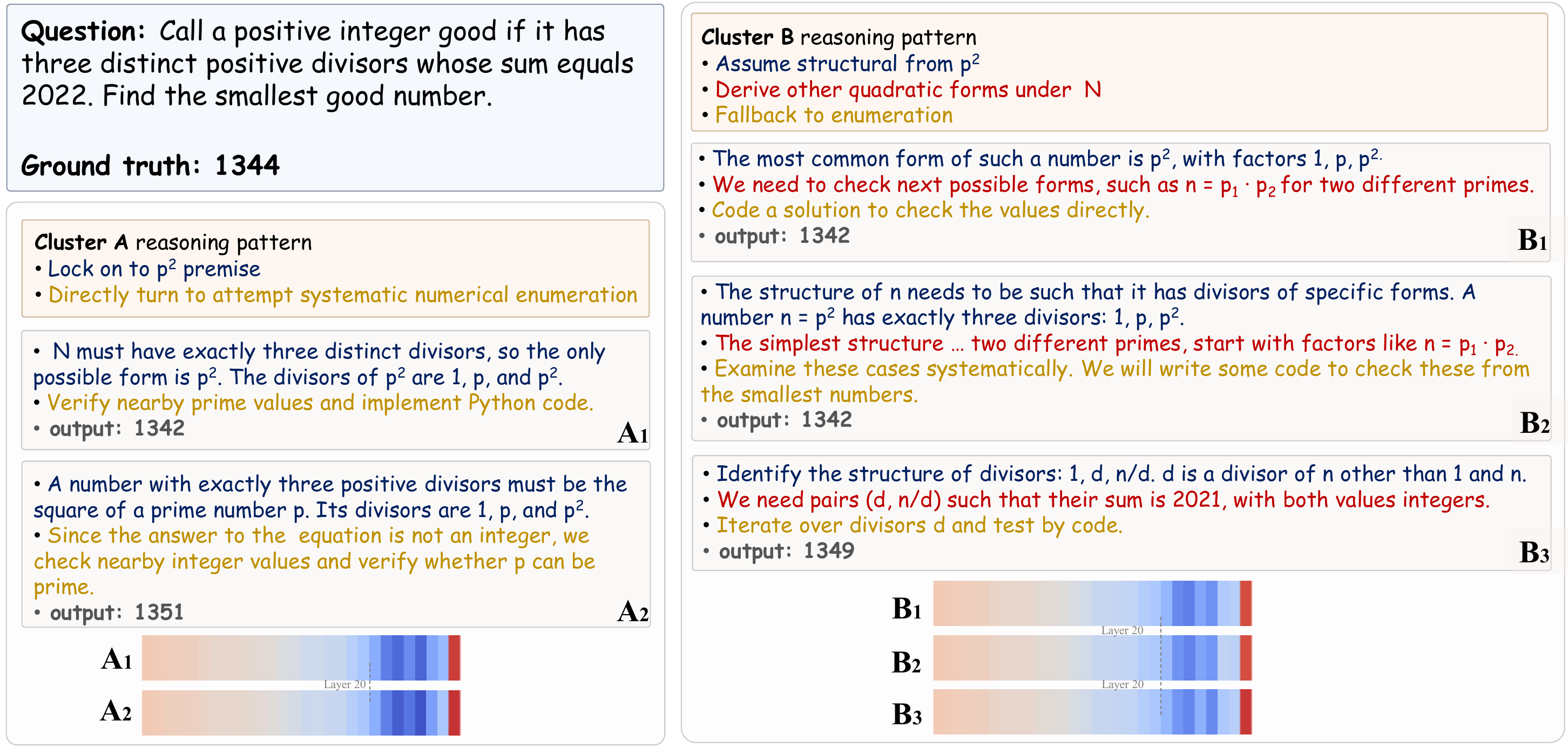}
\caption{
A case study. We manually annotate the reasoning structures of different responses. The responses are grouped into Cluster A containing $\mathbf{A_1, A_2}$ and Cluster B containing $\mathbf{B_1, B_2, B_3}$ by our clustering method. Although $\mathbf{A_1, B_1}$ and $\mathbf{B_2}$ produce the same incorrect answer 1342, their reasoning structures are different. A similar phenomenon can also be observed for the responses within Cluster B. We further visualize the logits using a heatmap. See Appendix~\ref{app: case_study} for the complete case study and the corresponding logit values.}
\label{fig:logit pattern}
\end{figure}

To demonstrate that logit vectors capture the underlying logical reasoning structure, we first perform a qualitative analysis of the logit vectors collected from rollouts across the \methodname~training process. This preliminary analysis suggests that responses with similar reasoning patterns tend to exhibit similar layer-wise logit trajectories across layers.

\paragraph{Logit patterns reflect reasoning structure}

Figure~\ref{fig:logit pattern} presents a case study example. The logits show relatively limited variation in the first 20 layers, where they primarily encode semantic information. In contrast, larger changes emerge in the later layers, where responses with similar reasoning patterns show clear correlations in their logit vectors, with layer-wise values tending to fall within similar ranges. Meanwhile, responses following different reasoning paths remain distinguishable. We further find that logit patterns reflect not only the final answer tokens, but also the intermediate reasoning process. Responses leading to the same answer may exhibit different logit patterns, while responses with different answers may still share similar reasoning behaviors.

\FloatBarrier
\subsection{Logits Provide Correct Clustering Signals}
\label{subsec: logit reflect}
To further validate this observation at scale, we introduce an automated annotation procedure using \textit{Claude-Haiku-4.5} as a proxy annotator. We randomly sample 800 negative responses from the entire dataset. Based on 11 common error categories distilled through manual inspection, we first ask the proxy model to assign each response to one of these error clusters. We then sample response pairs and compare whether the two responses belong to the same cluster according to both the LLM annotations and the logit-based clustering.
If the two judgments agree, i.e., both indicate that the responses belong to the same cluster or both indicate that they belong to different clusters, we regard the logit-based clustering as correct for that pair. The agreement rate is then used as the correlation metric. The manually defined error categories and the prompts used for annotation are provided in Appendix~\ref{app:prompt}.

\input{tables/ablation_results}

\paragraph{Logit clustering serves as an acceptable proxy for LLM-based clustering}
Figure~\ref{fig:eigen_ratio_claude_correlation} (right) shows the correlation between our logit-based clustering results and the annotations produced by Claude. With a correlation of 61.2\%, the logit clustering exhibits a substantial positive relationship with the logical similarity identified by the LLM annotator.

\subsection{Clustering Quality Matters}
\label{subsec:cluster_quality}

Building on the rich representational signals provided by logits, clustering features can be constructed in multiple ways. We investigate how different feature construction choices affect clustering quality and how these differences further influence downstream evaluation performance.

Specifically, we consider several feature construction strategies that vary in the selected layers and in whether logit differencing is applied.
As a worst-case baseline, we also include a degenerate setting in which all responses are assigned to a single cluster, yielding a trivial clustering with virtually no discriminative power. All experiments use Qwen2.5-Math-7B and differ only in the choice of clustering features. Following the Claude-correlation evaluation protocol described in Section~\ref{subsec:logits_pattern}, we report the results of different logit aggregation strategies in Figure~\ref{fig:eigen_ratio_claude_correlation} (right).

\paragraph{Using the last 14 layers as clustering features yields the best downstream performance.}
Table~\ref{tab:layer-selection} reports the results under different feature construction settings. Among all configurations, using the last 14 layers as clustering features yields the best performance. In contrast, excluding logits altogether results in performance worse than DAPO, while every variant that leverages clustering information outperforms the DAPO baseline. This provides further evidence that logits serve as effective historical signals for reward shaping.

\paragraph{Clustering quality strongly correlates with downstream performance.}
Figure~\ref{fig:eigen_ratio_claude_correlation} (right) shows that using the last 14 layers as clustering features achieves the highest agreement with the Claude annotations, whereas the degenerate single-cluster setting yields the lowest.
More importantly, the ranking of clustering consistency closely mirrors the ranking of downstream performance, suggesting that better clustering yields larger downstream improvements. This provides further evidence for the effectiveness of our approach.

%% file: tables/main_results.tex
\begin{table*}[t]
\centering
\small
\setlength{\tabcolsep}{3pt}
\renewcommand{\arraystretch}{1.3}

\resizebox{\textwidth}{!}{
\begin{tabular}{l c*{12}{c}}
\toprule
\multicolumn{1}{c}{\multirow{2}{*}{\textbf{Method}}} &
\multicolumn{2}{c}{\textbf{AIME24}} &
\multicolumn{2}{c}{\textbf{AMC23}} &
\multicolumn{2}{c}{\textbf{MATH500}} &
\multicolumn{2}{c}{\textbf{Minerva}} &
\multicolumn{2}{c}{\textbf{OlympiadBench}} &
\multicolumn{2}{c}{\textbf{Average}} \\
& \textbf{pass@1} & \textbf{pass@128} &
\textbf{pass@1} & \textbf{pass@128} &
\textbf{pass@1} & \textbf{pass@128} &
\textbf{pass@1} & \textbf{pass@128} &
\textbf{pass@1} & \textbf{pass@128} &
\textbf{pass@1}  & \textbf{pass@128} \\
\midrule

Qwen3-1.7B & \second{22.58} & \second{66.67} & \second{58.57} & 87.50 & \second{80.86} & 95.20 & 37.82 & 63.6 & 41.97 & 64.15 & \second{48.36} & 75.42 \\

+GRPO & 18.20 & 46.67 & 51.58 & \second{92.50} & 79.58 & 95.60 & 38.62 & 65.81 & 40.19 & 70.37 & 45.63 & 74.19 \\

+DAPO & 13.46 & 63.33 & 57.65 & \best{95.00} & 80.25 & \second{96.60} & \second{39.45} & \best{69.49} & \second{42.20} & \second{77.04} & 46.6 & \second{80.29} \\

+GRPO w/ Entropy Adv & 12.50 & 50.00 & 51.80 & \best{95.00} & 78.72 & 94.00 & 38.91 & 62.87 & 38.47 & 66.96 & 44.08 & 73.77 \\

\rowcolor{gray!14}
\textbf{+\methodname} & \best{23.65} & \best{70.00} & \best{65.51} & \best{95.00} & \best{84.66} & \best{97.40} & \best{41.77} & \second{68.75} & \best{46.86} & \best{77.63} & \best{52.49} & \best{81.76} \\

\midrule

Qwen2.5-Math-7B & 6.43 & 56.67 & 19.82 & \second{92.50} & 30.84 & \second{96.00} & 6.62 & 63.24 & 9.27 & 70.07 & 14.60 & 75.70 \\

+GRPO & 23.85 & 66.67 & 70.08 & \best{97.50} & 84.10 & \best{96.20} & 39.07 & 73.16 & 42.57 & 73.63 & 51.93 & 81.43 \\

+DAPO & \second{32.27} & \second{70.00} & \second{72.71} & \best{97.50} & \second{85.61} & \second{96.00} & \best{43.01} & \best{75.74} & \second{44.43} & \second{75.00} & \second{55.61} & \second{82.85} \\

+GRPO w/ Entropy Adv & 29.74 & 63.33 & 70.64 & \best{97.50} & 84.50 & \best{96.20} & 40.68 & 70.96 & 42.57 & 73.63 & 53.63 & 80.32 \\

\rowcolor{gray!14}\textbf{+\methodname} & \best{34.32} & \best{76.67} & \best{74.38} & \best{97.50} & \best{86.33} & \second{96.00} & \second{42.51} & \second{74.26} & \best{44.80} & \best{75.56} & \best{56.47} & \best{84.00} \\

\midrule

Qwen3-8B & 34.51 & 70.00 & 65.72 & 90.00 & 83.62 & 95.40 & 40.90 & 61.03 & 44.69 & 61.63 & 53.89 & 75.61 \\

+GRPO & 22.45 & 70.00 & 58.79 & 90.00 & 82.38 & 95.40 & 44.07 & 64.71 & 44.22 & 69.78 & 50.38 & 77.98 \\

+DAPO & \second{45.42} & \second{73.33} & \second{81.37} & \second{95.00} & \second{89.18} & \second{96.60} & \second{46.82} & \second{65.44} & \second{52.77} & \second{70.81} & \second{63.11} & \second{80.24} \\

+GRPO w/ Entropy Adv & 27.32 & 63.33 & 65.49 & 92.50 & 85.89 & 96.00 & 45.32 & \second{65.44} & 49.92 & 70.67 & 54.79 & 77.59 \\

\rowcolor{gray!14}\textbf{+\methodname} & \best{45.78} & \best{76.67} & \best{82.62} & \best{97.50} & \best{92.51} & \best{98.20} & \best{51.58} & \best{68.02} & \best{61.12} & \best{82.67} & \best{66.72} & \best{84.61} \\

\bottomrule
\end{tabular}
}
\caption{Main results on five mathematical reasoning benchmarks. Best and second-best are marked with \best{bold} and \second{underline}; \methodname{} rows are highlighted.}
\label{tab:main_results}
\end{table*}

%% file: tables/ablation_results.tex
\begin{table*}[t]
\centering
\setlength{\tabcolsep}{3pt}
\renewcommand{\arraystretch}{1.3}
\resizebox{\textwidth}{!}{%
\begin{tabular}{lccccccccccccc}
\hline
\multicolumn{1}{c}{\multirow{2}{*}{\textbf{Method}}} & \multicolumn{2}{c}{\textbf{AIME24}} & \multicolumn{2}{c}{\textbf{AMC}} & \multicolumn{2}{c}{\textbf{MATH500}} & \multicolumn{2}{c}{\textbf{Minerva}} & \multicolumn{2}{c}{\textbf{OlympiadBench}} & \multicolumn{2}{c}{\textbf{Average}} \\
\multicolumn{1}{c}{} & \textbf{pass@1} & \textbf{pass@128} & \textbf{pass@1} & \textbf{pass@128} & \textbf{pass@1} & \textbf{pass@128} & \textbf{pass@1} & \textbf{pass@128} & \textbf{pass@1} & \textbf{pass@128} & \textbf{pass@1} & \textbf{pass@128} \\ \hline
DAPO & 32.27 & 70.00 & 72.71 & 97.50 & 85.61 & 96.00 & 43.01 & 75.74 & 44.43 & 75.00 & 55.61 & 82.85 \\
\methodname-single cluster & 29.70 & 70.00 & 73.60 & 100.00 & 86.60 & 96.80 & 41.20 & 71.69 & 44.60 & 74.52 & 55.14 & 82.60 \\
\methodname-28-diff & 31.88 & 70.00 & 72.34 & 97.50 & 85.80 & 96.00 & 43.38 & 73.53 & 45.48 & 74.37 & 55.78 & 82.28 \\
\methodname-14-diff & 32.40 & 73.33 & 71.64 & 100.00 & 86.40 & 96.60 & 43.01 & 72.43 & 44.15 & 74.52 & 55.52 & 83.38 \\
\methodname-28 & 35.63 & 70.00 & 73.13 & 97.50 & 86.20 & 96.40 & 41.90 & 73.16 & 44.00 & 73.78 & 56.17 & 82.17 \\
\rowcolor[HTML]{EDEDED} 
\methodname-14 & 34.32 & 76.67 & 74.38 & 97.50 & 86.33 & 96.00 & 42.51 & 74.26 & 44.80 & 75.56 & 56.47 & 84.00 \\ \hline
\end{tabular}%
}
\caption{Comparison of different clustering feature constructions on Qwen2.5-Math-7B, including assigning all logits to a single cluster (\textit{single cluster}); using adjacent logit differences over 28 layers (\textit{28-diff}) or the last 14 layers (\textit{14-diff}); using logits from all 28 layers; and using only the last 14 layers. The first row corresponds to the vanilla DAPO baseline.}

\label{tab:layer-selection}
\end{table*}

%% file: sections/conclusion.tex
\section{Conclusion}
In this work, we study the recurring error patterns that emerge during reinforcement learning for LLM reasoning. We introduce \methodname, a memory-aware reward shaping framework that dynamically penalizes recurrent failure patterns. We first provide theoretical analysis demonstrating the effectiveness of our method. By reusing layer-wise logits as lightweight representations of reasoning trajectories, our approach clusters historical responses to identify and suppress repeated error modes. Empirical results across multiple benchmarks show that \methodname~consistently improves both reasoning performance and exploration diversity.
These findings highlight the importance of incorporating historical trajectory information into reward design and open new directions for improving exploration in LLM reinforcement learning.

%% file: sections/limitations.tex
\section{Limitations}
The main limitation of our study is that the methods explored for utilizing logits are relatively simple and do not incorporate more sophisticated aggregation functions. It remains unclear whether adopting more complex aggregation strategies could further improve performance.

%% file: sections/appendix.tex
\section{Implementation Details}
\label{sec:appendix}

\subsection{Qwen-Math Template}
\label{subsec:qwen-math-template}
All problems in our training and evaluation datasets are processed using the Qwen-Math template, which provides a standardized format for mathematical problem-solving prompts.

\begin{verbatim}
<|im_start|>system
Please reason step by step, and put your 
final answer within \boxed{}.<|im_end|>
<|im_start|>user
{problem}
<|im_end|>
<|im_start|>assistant
\end{verbatim}

\subsection{Training Hyperparameters and Objectives}
\label{subsec:training-hyperparams}

All methods are trained with the Adam optimizer at a learning rate of $1\times10^{-6}$, weight decay $0.1$, gradient clipping threshold $1.0$, and $10$ warmup steps, with a PPO mini-batch size of $16$.

\textbf{GRPO} follows the official DeepSeekMath configuration~\citep{grpo}, with clipping ratio $\epsilon=0.2$ and KL penalty coefficient $\beta=10^{-3}$.

\textbf{DAPO} employs decoupled clipping with asymmetric thresholds $\varepsilon_{\text{low}}=0.2$ and $\varepsilon_{\text{high}}=0.28$, and removes the KL penalty term entirely.

\textbf{GRPO w/ Entropy Adv} follows \citet{Cheng2025ReasoningWE}, augmenting the group-relative advantage with a per-token entropy bonus weighted by $\alpha$.

The training objectives are defined as follows. 
\begin{equation}
\begin{aligned}
&\mathcal{L}_{\text{GRPO}}(\theta) =
 \mathbb{E}_{q \sim \mathcal{D}} \Bigg[ \frac{1}{K} \sum_{i=1}^{K} \min \Bigg( r_i(\theta) A_i,\mathrm{clip}(r_i(\theta), 1{-}\epsilon, 1{+}\epsilon) A_i \Bigg) \Bigg] - \beta\, \mathbb{D}_{\mathrm{KL}}\!\left[\pi_\theta \,\|\, \pi_{\mathrm{ref}}\right]
\end{aligned}
\end{equation}

\begin{equation}
\begin{aligned}
\mathcal{L}_{\text{DAPO}}(\theta) 
  &= \mathbb{E}_{q \sim \mathcal{D}} \Bigg[ \frac{1}{K} \sum_{i=1}^{K} \min \Bigg( r_i(\theta) A_i,\mathrm{clip}(r_i(\theta), 1{-}\epsilon_{\text{low}}, 1{+}\epsilon_{\text{high}}) A_i \Bigg) \Bigg]
\end{aligned}
\end{equation}

\begin{equation}
\begin{aligned}
&\mathcal{L}_{\text{GRPO+Ent}}(\theta) 
= \mathbb{E}_{q \sim \mathcal{D}} \Bigg[ \frac{1}{K} \sum_{i=1}^{K} \min \Bigg( r_i(\theta) \tilde{A}_i,\mathrm{clip}(r_i(\theta), 1{-}\epsilon, 1{+}\epsilon) \tilde{A}_i \Bigg) \Bigg] - \beta\, \mathbb{D}_{\mathrm{KL}}\!\left[\pi_\theta \,\|\, \pi_{\mathrm{ref}}\right]
\end{aligned}
\end{equation}
where the probability ratio, group-relative advantage, entropy-augmented advantage, and per-token entropy are respectively defined as:
\begin{equation}
r_i(\theta) = \frac{\pi_\theta(o_i \mid q)}{\pi_{\theta_{\text{old}}}(o_i \mid q)}
\end{equation}
\begin{equation}
A_i = \frac{R_i - \text{mean}(\{R_j\}_{j=1}^{K})}{\text{std}(\{R_j\}_{j=1}^{K})}
\end{equation}
\begin{equation}
\tilde{A}_i = A_i + \alpha \cdot \mathcal{H}_t
\end{equation}
\begin{equation}
\mathcal{H}_t = -\sum_{v \in \mathcal{V}} \pi_\theta(v \mid q, o_{<t}) \log \pi_\theta(v \mid q, o_{<t})
\end{equation}



\section{Theoratical Proof}
\label{subsec:theoratical_proof}
\begin{theorem}
Suppose there are two different reward signals, $\mu_1 = r(\mathbf{x}, \mathbf{y})$ and $\mu_2 = r(\mathbf{x}, \mathbf{y}) - \lambda c(\mathbf{y})$, and let $q_1$ and $q_2$ denote the corresponding updated policies under these reward signals. Let $J(q) = \mathbb{E}_{\mathbf{x}, \mathbf{y} \sim q}[r(\mathbf{x}, \mathbf{y})]$. Then it holds that
\begin{align}
J(q_2) \geq J(q_1),
\end{align}
where $\lambda$ is the repetition penalty coefficient. $c(\mathbf{y}): \mathbf{x} \times Y \rightarrow \mathbb{N}$ denotes the indicator of the number of repeated errors, where a value greater than zero indicates that $\mathbf{y}$ corresponds to an error that has been repeatedly sampled in the past. The function $r(\mathbf{x}, \mathbf{y})$ is a task-specific scoring function.
\end{theorem}
\begin{proof}
For clarity we work pointwise in $\mathbf{x}$ and assume $\mathcal{Y}$ is finite (the continuous case follows by replacing sums with integrals).
According to the design that imposes additional penalties on repeated errors, we have:
\begin{equation}
c(\mathbf{y_1}) \ge c(\mathbf{y_2}) \;\Longrightarrow\; r(\mathbf{x}, \mathbf{y_1}) \le r(\mathbf{x}, \mathbf{y_2}).
\label{eq:repeat_penalty}
\end{equation}
That is, responses with larger repetition count are not higher-scoring under the task reward.

For any scoring signal $\mu(\mathbf{x},\mathbf{y})$, define the KL-regularized one-step update (for each $\mathbf{x}$) as
\[
\begin{aligned}
q_\mu(\cdot\mid \mathbf{x})
\in &\arg\max_{q(\cdot\mid \mathbf{x})}
\Big\{ 
\mathbb{E}_{\mathbf{y}\sim q(\cdot\mid \mathbf{x})}[\mu(\mathbf{x},\mathbf{y})] - \frac{1}{\eta}
\mathrm{KL}(q(\cdot\mid \mathbf{x})\,\|\,p_{\mathrm{ref}}(\cdot\mid \mathbf{x}))
\Big\}.
\end{aligned}
\]
Using a Lagrange multiplier for the constraint $\sum_\mathbf{y} q(\mathbf{y}\mid \mathbf{x})=1$, the optimizer has the Gibbs form\citep{DPO}
\begin{equation}
\begin{aligned}
q_\mu(\mathbf{y}\mid \mathbf{x})
&= \frac{p_{\mathrm{ref}}(\mathbf{y}\mid \mathbf{x})\exp(\eta\,\mu(\mathbf{x},\mathbf{y}))}{Z_\mu(\mathbf{x})}, \\
Z_\mu(\mathbf{x})
&= \sum_{\mathbf{y}} p_{\mathrm{ref}}(\mathbf{y}\mid \mathbf{x})\exp(\eta\,\mu(\mathbf{x},\mathbf{y})).
\end{aligned}
\label{eq:gibbs}
\end{equation}
Let $\mu_1(\mathbf{x},\mathbf{y})=r(\mathbf{x},\mathbf{y})$ and $\mu_2(\mathbf{x},\mathbf{y})=r(\mathbf{x},\mathbf{y})-\lambda c(\mathbf{y})$ with $\lambda>0$, and denote $q_1=q_{\mu_1}$, $q_2=q_{\mu_2}$.

From \eqref{eq:gibbs},
\[
\begin{aligned}
q_1(\mathbf{y}\mid \mathbf{x})
&= \frac{p_{\mathrm{ref}}(\mathbf{y}\mid \mathbf{x})\exp(\eta r(\mathbf{x},\mathbf{y})}{Z_1(\mathbf{x})}, \\
q_2(\mathbf{y}\mid \mathbf{x})
&= \frac{p_{\mathrm{ref}}(\mathbf{y}\mid \mathbf{x})\exp(\eta r(\mathbf{x},\mathbf{y}) - \eta\lambda c(\mathbf{y}))}{Z_2(\mathbf{x})}.
\end{aligned}
\]
Define $w(\mathbf{y}):=\exp(-\eta\lambda c(\mathbf{y}))\in(0,1]$. Then for each fixed $\mathbf{x}$,
\begin{equation}
\begin{aligned}
q_2(\mathbf{y}\mid \mathbf{x})
&= \frac{q_1(\mathbf{y}\mid \mathbf{x})\,w(\mathbf{y})}
{\sum_{\mathbf{y}'} q_1(\mathbf{y}'\mid \mathbf{x})\,w(\mathbf{y}')} \\
&= \frac{q_1(\mathbf{y}\mid \mathbf{x})\,w(\mathbf{y})}
{\mathbb{E}_{\mathbf{y}\sim q_1(\cdot\mid \mathbf{x})}[w(\mathbf{y})]}.
\end{aligned}
\label{eq:reweight}
\end{equation}
Define the per-$\mathbf{x}$ expected reward under policy $q$ as
\[
J_\mathbf{x}(q) := \mathbb{E}_{\mathbf{y}\sim q(\cdot\mid \mathbf{x})}[r(\mathbf{x},\mathbf{y})] .
\]

Using \eqref{eq:reweight}, we obtain
\begin{equation}
\begin{aligned}
J_\mathbf{x}(q_2)
&= \sum_\mathbf{y} q_2(\mathbf{y}\mid \mathbf{x}) r(\mathbf{x},\mathbf{y}) \\
&= \frac{\sum_\mathbf{y} q_1(\mathbf{y}\mid \mathbf{x})\,w(\mathbf{y})\,r(\mathbf{x},\mathbf{y})}
        {\sum_\mathbf{y} q_1(\mathbf{y}\mid \mathbf{x})\,w(\mathbf{y})} \\
&= \frac{\mathbb{E}_{\mathbf{y}\sim q_1(\cdot\mid \mathbf{x})}[w(\mathbf{y})\,r(\mathbf{x},\mathbf{y})]}
        {\mathbb{E}_{\mathbf{y}\sim q_1(\cdot\mid \mathbf{x})}[w(\mathbf{y})]}.
\end{aligned}
\label{eq:Jratio}
\end{equation}

Hence,
\begin{equation}
\begin{aligned}
J_\mathbf{x}(q_2) - J_\mathbf{x}(q_1)
&= \frac{\mathbb{E}_{q_1}[w\,r] - \mathbb{E}_{q_1}[w]\mathbb{E}_{q_1}[r]}
        {\mathbb{E}_{q_1}[w]} \\
&= \frac{\mathrm{Cov}_{q_1}(r(\mathbf{x},\mathbf{y}), w(\mathbf{y}))}
        {\mathbb{E}_{q_1}[w(\mathbf{y})]}.
\end{aligned}
\label{eq:cov}
\end{equation}

Since $\mathbb{E}_{q_1}[w(\mathbf{y})] > 0$, it suffices to show that
$\mathrm{Cov}_{q_1}(r(\mathbf{x},\mathbf{y}), w(\mathbf{y})) \ge 0$.

Let $Y \sim q_1(\cdot\mid \mathbf{x})$ and $C = c(Y)$. Note that
$w(Y) = \exp(-\eta\lambda C)$ is non-increasing in $C$.
By equation \eqref{eq:repeat_penalty}, $r(\mathbf{x},Y)$ is also non-increasing in $C$
(equivalently, the conditional mean $\bar r_\mathbf{x}(k) := \mathbb{E}[r(\mathbf{x},Y)\mid C=k]$
is non-increasing in $k$).

Therefore, by Chebyshev's rearrangement inequality,
\begin{equation}
\begin{aligned}
\mathbb{E}[\bar r_\mathbf{x}(C)\exp(-\eta\lambda C)]
\;\ge\;&
\mathbb{E}[\bar r_\mathbf{x}(C)] \cdot\mathbb{E}[\exp(-\eta\lambda C)] .
\end{aligned}
\end{equation}

Using the tower property,
\begin{equation}
\begin{aligned}
&\mathbb{E}_{q_1}[r(\mathbf{x},Y)w(Y)]\\
&= \mathbb{E}\!\left[
\mathbb{E}[r(\mathbf{x},Y)\mid C]\exp(-\eta\lambda C)
\right] \\
&= \mathbb{E}[\bar r_\mathbf{x}(C)\exp(-\eta\lambda C)] .
\end{aligned}
\end{equation}

Similarly,
\begin{equation}
\begin{aligned}
\mathbb{E}_{q_1}[r(\mathbf{x},Y)]
&= \mathbb{E}[\bar r_\mathbf{x}(C)], \\
\mathbb{E}_{q_1}[w(Y)]
&= \mathbb{E}[\exp(-\eta\lambda C)] .
\end{aligned}
\end{equation}

Thus, $\mathrm{Cov}_{q_1}(r(\mathbf{x},Y), w(Y)) \ge 0$, and \eqref{eq:cov} implies $J_\mathbf{x}(q_2) \ge J_\mathbf{x}(q_1)$ for all $\mathbf{x}$.

Taking expectation over $\mathbf{x} \sim D$, we obtain
\begin{equation}
\begin{aligned}
J(q_2)
&= \mathbb{E}_{\mathbf{x}\sim D}[J_\mathbf{x}(q_2)] \\
&\ge \mathbb{E}_{\mathbf{x}\sim D}[J_\mathbf{x}(q_1)] \\
&= J(q_1).
\end{aligned}
\end{equation}

This completes the proof.
\end{proof}

\newpage
\section{Full Results of Case Study}

\label{app: case_study}

\begin{tcolorbox}[breakable, enhanced, colback=white, colframe=black, boxrule=0.4pt, width=\columnwidth, enlarge left by=0mm, enlarge right by=0mm, left=2pt, right=2pt, fontupper=\small, title=\textbf{Question}, colbacktitle=gray!20, coltitle=black]
Call a positive integer \textit{good} if it has three distinct positive divisors whose sum equals 2022. Find the smallest good number.

\smallskip
\textbf{Ground truth: 1344}
\end{tcolorbox}

\subsection{Logit Values}
Figure~\ref{fig:heatmapwithnum} is the logit heatmap in Figure~\ref{fig:logit pattern}, which has logit values on it.

\begin{figure}[!htbp]
    \centering
\includegraphics[width=0.82\linewidth]{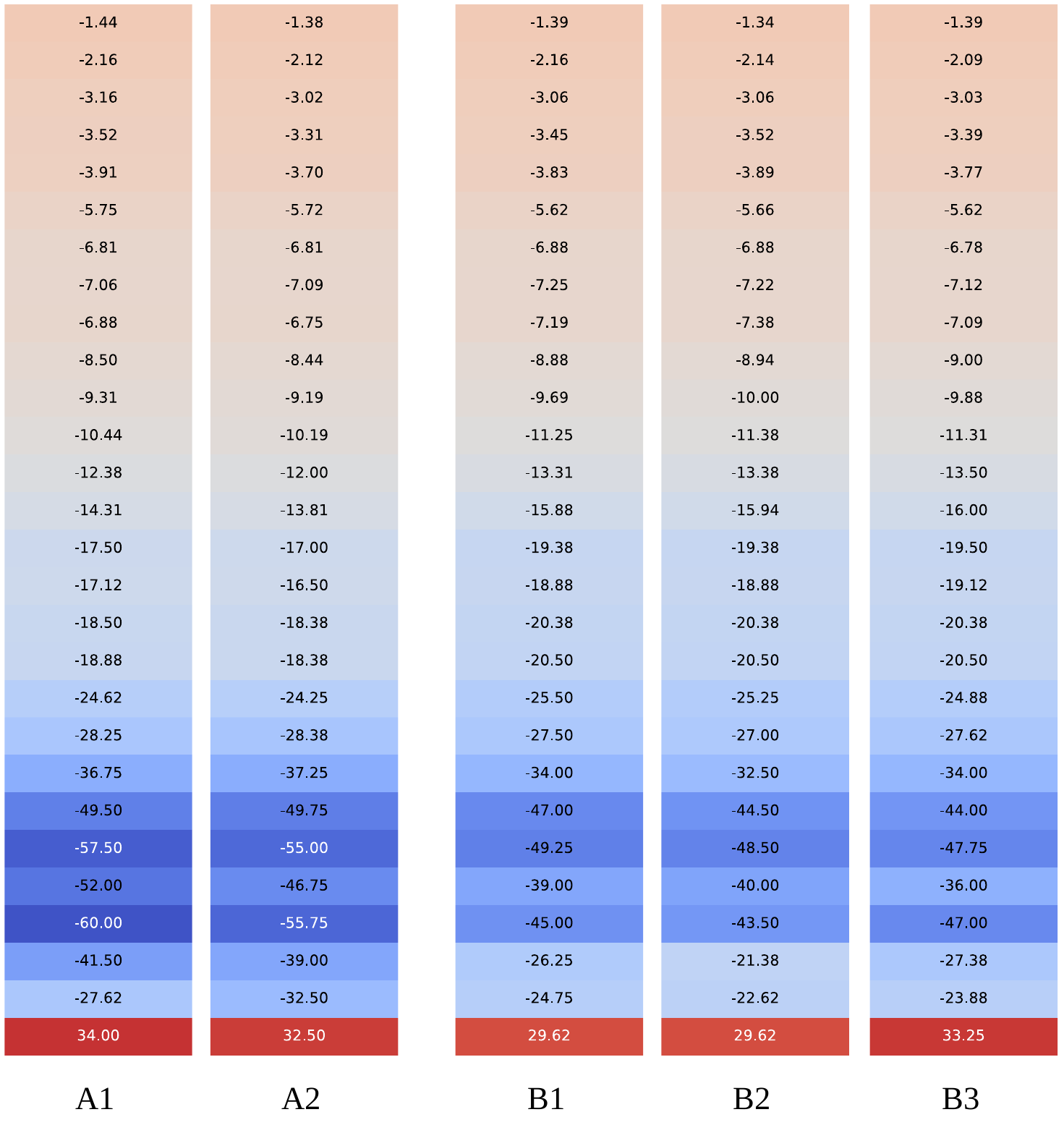}
    \caption{Heatmap in Figure~\ref{fig:logit pattern} with logit values.}
    \label{fig:heatmapwithnum}
\end{figure}

\subsection{Custer A Responses}

\begin{tcolorbox}[breakable, enhanced, colback=white, colframe=black, boxrule=0.4pt, width=\columnwidth, enlarge left by=0mm, enlarge right by=0mm, left=2pt, right=2pt, fontupper=\small]
\textbf{Response A1} \hfill \textit{Output: 1342}

\smallskip
To solve the problem of finding the smallest good number, we need to understand that a number is good if it has three distinct positive divisors that sum to 2022.

A key insight here is that if a number $n$ is a good number, then it must have exactly three distinct divisors. The only type of number that has exactly three distinct divisors is a square of a prime number. This is because the number of divisors of a number $n = p_1^{e_1} p_2^{e_2} \cdots p_k^{e_k}$ is given by $(e_1 + 1)(e_2 + 1) \cdots (e_k + 1)$. To have exactly 3 divisors, the only possible form is $p^2$ because $3$ can only be factored as $3 = 3 \times 1$.

The divisors of $p^2$ are $1, p, p^2$. We need these to sum to 2022, so we need:
\[
1 + p + p^2 = 2022
\]
This simplifies the equation to:
\[
p^2 + p + 1 = 2022
\]
Rearranging the terms, we get a standard quadratic equation:
\[
p^2 + p - 2021 = 0
\]
We will solve this quadratic equation using the quadratic formula $p = \frac{-b \pm \sqrt{b^2 - 4ac}}{2a}$:
\begin{align*}
p &= \frac{-1 \pm \sqrt{1 + 4 \times 2021}}{2} \\
  &= \frac{-1 \pm \sqrt{8085}}{2}
\end{align*}
We are only interested in the positive value of $p$, so we need to check if $\sqrt{8085}$ is an integer. We will calculate $\sqrt{8085}$ and check its value. If $8085$ is a perfect square, we can proceed further.

By calculating, we find that $\sqrt{8085} \approx 89.9166$, which is not an integer. So we need to find the closest prime number. We check the closest perfect square value and determine:
\[
89^2 = 7921 \quad \text{and} \quad 90^2 = 8100
\]
Since $8085$ lies between these two, we check the values for the primes near 89 and verify. We check whether we can substitute $89$ into our equation to see if we get a prime number for $p$:

If we let $p = 89$: $89^2 + 89 + 1 = 7921 + 89 + 1 = 8011 \neq 2022$.

So we will try $p = 83$: $83^2 + 83 + 1 = 6889 + 83 + 1 = 6973 \neq 2022$.

This also does not match. We will check $p = 83$. Let's use a closest prime number which might satisfy, let's take $p = 43$: $43^2 + 43 + 1 = 1849 + 43 + 1 = 1893 \neq 2022$.

This does not match either. Let's check with smaller primes. We will check $p = 41$: $41^2 + 41 + 1 = 1681 + 41 + 1 = 1723 \neq 2022$.

This is still not 2022. Let's check with $p = 47$: $47^2 + 47 + 1 = 2209 + 47 + 1 = 2257 \neq 2022$.

This also does not work. Now we check $p = 43$: $43^2 = 1849$, $1849 + 43 + 1 = 1893 \neq 2022$.

Let's check with larger prime numbers and verify by code to ensure we get the exact result. We can implement a Python code to check the values and find the smallest $n$.

We will start by checking $p = 45$ as a possible value: $45^2 = 2025$, which is greater than 2022, but let's proceed and confirm by code.

We need to find $p$ such that $p + 1 + p^2 = 2022$. Let's implement this code to find the exact prime number.
\begin{Verbatim}[fontsize=\footnotesize]
import math

def find_smallest_good_number():
    # We will start with the smallest possible 
    # number and check upwards
    num = 1
    while True:
        divisors = [d for d in range(1, \
                    num + 1) if num % d == 0]
        if len(divisors) >= 3:
            divisors = sorted(divisors)[:3]
            if sum(divisors) == 2022:
                return num
        num += 1

smallest_good_number = \
                find_smallest_good_number()
print(smallest_good_number)
\end{Verbatim}
\texttt{Output: 1342}

\smallskip
The smallest good number, which satisfies the condition that three distinct positive divisors sum to 2022, is $\boxed{1342}$.
\end{tcolorbox}

\begin{tcolorbox}[breakable, enhanced, colback=white, colframe=black, boxrule=0.4pt, width=\columnwidth, enlarge left by=0mm, enlarge right by=0mm, left=2pt, right=2pt, fontupper=\small]
\textbf{Response A2} \hfill \textit{Output: 1351}

\smallskip
To find the smallest number with exactly three distinct positive divisors that sum up to 2022, we need to analyze the structure of such a number. A number with exactly three positive divisors must be the square of a prime number $p$, because its divisors are 1, $p$, and $p^2$.

Given the condition, if $n = p^2$, then we need $1 + p + p^2 = 2022$. We can set up the equation:
\[
p^2 + p + 1 = 2022.
\]
Rearranging gives:
\[
p^2 + p - 2021 = 0.
\]
This is a quadratic equation of the form $p^2 + p - 2021 = 0$. We can solve for $p$ using the quadratic formula $p = \frac{-b \pm \sqrt{b^2 - 4ac}}{2a}$, where $a = 1$, $b = 1$, and $c = -2021$. This gives:
\begin{align*}
p &= \frac{-1 \pm \sqrt{1^2 - 4 \cdot 1 \cdot (-2021)}}{2 \cdot 1} \\
  &= \frac{-1 \pm \sqrt{1 + 8084}}{2} \\
  &= \frac{-1 \pm \sqrt{8085}}{2}.
\end{align*}
We are only interested in the positive root, so we compute:
\[
p = \frac{-1 + \sqrt{8085}}{2}.
\]
We need to check if $\sqrt{8085}$ is an integer, and if $\frac{-1 + \sqrt{8085}}{2}$ is a prime number.

First, let's approximate $\sqrt{8085}$: $\sqrt{8085} \approx 89.9166$. So, $\sqrt{8085}$ is very close to 90. Let's check if 90 is indeed the square root: $90^2 = 8100$, which is slightly larger than 8085. So, $\sqrt{8085} \approx 89.9166 \approx 90$ (not exact, so we check 89).

Let's check 89 directly: $89^2 = 7921$, which is smaller than 8085. Let's check 90: $90^2 = 8100$, which is slightly larger than 8085. Thus, 89 is closer and we need to check it more precisely:
\[
89^2 = 7921 \quad \text{and} \quad 90^2 = 8100.
\]
Since $89^2$ is closer to 8085, we use it: $\sqrt{8085} \approx 89.9166$, and thus:
\[
p = \frac{-1 + 89.9166}{2} \approx \frac{88.9166}{2} \approx 44.4583,
\]
which isn't an integer. Let's check the closest integer values. We need $p$ to be an integer and prime. Let's check if 89 is a solution by substituting it back:
\begin{align*}
&p^2 + p - 2021 = 89^2 + 89 - 2021 = 7921 + 89 - 2021 \\
&= 4078 \neq 0.
\end{align*}

Thus, we check the next closest prime number, and we know that $p$ must be an integer. So let's check manually for the factorization and see if we can find a prime $p$:

If we try solving the equation directly for integer values:
\[
p^2 + p - 2021 = 0 \implies p(p+1) = 2021.
\]
We look for pairs of factors of 2021. We factorize 2021:
\[
2021 = 43 \times 47.
\]
We see that 43 and 44 don't work, but checking 43 directly in the equation: $43^2 = 1849$ and $43 + 1849 = 1892 \neq 2022$.

Let's check if there's a smaller $p$ that works by iterating through possible prime values. We use Python to automate checking for the smallest $n$. We'll start by finding potential $p$ values that make $1 + p + p^2 = 2022$ a valid equation and check if $p$ is prime.
\begin{Verbatim}[fontsize=\footnotesize]
import math

def is_prime(n):
    if n <= 1:
        return False
    if n <= 3:
        return True
    if n % 2 == 0 or n % 3 == 0:
        return False
    i = 5
    while i * i <= n:
        if n % i == 0 or n % (i + 2) == 0:
            return False
        i += 6
    return True

def find_smallest_good_number():
    target_sum = 2022
    p = 1
    while True:
        p_squared = p * p
        if 1 + p + p_squared == target_sum:
            if is_prime(p):
                return p_squared
        p += 1

smallest_good_number = \
                find_smallest_good_number()
print(smallest_good_number)
\end{Verbatim}
\texttt{Output: 1351}

\smallskip
The smallest number that has exactly three distinct positive divisors summing up to 2022 is $\boxed{1351}$.
\end{tcolorbox}

\subsection{Cluster B Responses}

\begin{tcolorbox}[breakable, enhanced, colback=white, colframe=black, boxrule=0.4pt, width=\columnwidth, enlarge left by=0mm, enlarge right by=0mm, left=2pt, right=2pt, fontupper=\small]
\textbf{Response B1} \hfill \textit{Output: 1342}

\smallskip
To find the smallest positive integer $n$ that has exactly three distinct positive factors whose sum equals 2022, we need to analyze the structure of the factors.

\textbf{Analysis:}

\textbf{1. Understanding the Factors:} For a number $n$ to have exactly three distinct positive factors, the most common form of such a number is $p^2$, where $p$ is a prime number. The factors of $p^2$ are $1, p$, and $p^2$.

So, if $n = p^2$, the sum of the factors $1 + p + p^2 = 2022$. Therefore, we have:
\[
p^2 + p + 1 = 2022 \implies p^2 + p - 2021 = 0
\]

\textbf{2. Solving the Quadratic Equation:} We solve for $p$ using the quadratic formula $p = \frac{-b \pm \sqrt{b^2 - 4ac}}{2a}$ where $a = 1, b = 1$, and $c = -2021$:
\begin{align*}
p &= \frac{-1 \pm \sqrt{1 + 4 \cdot 2021}}{2} \\
  &= \frac{-1 \pm \sqrt{8085}}{2}
\end{align*}
Since $p$ is a prime number, we need to check if $\sqrt{8085}$ is an integer and if $\frac{-1 + \sqrt{8085}}{2}$ is a prime number. We find that $\sqrt{8085} = 89.917$. This means the closest integer value for the square root is 89. Checking if 89 works:
\[
p = \frac{-1 + 89}{2} = \frac{88}{2} = 44
\]
But 44 is not a prime number. We need to check the next possible forms. Another form for having three factors might be when $n = p_1 \cdot p_2$ for two different primes $p_1$ and $p_2$. However, since $p_1 \cdot p_2$ would suggest factors $1, p_1, p_2$, and $p_1 p_2$, having exactly three factors requires a form without additional factors. Instead, let's examine the simplest form more closely. Since $44$ wasn't a prime, we incrementally check the next prime values for $p$:

\textbf{Iterating through Prime Numbers:}

We should check if the factor forms might lead to other structures that fit, or directly check if smaller values fit our sum of three factors form directly.

Let's code a solution to check the values directly: iterate through numbers calculating their factors and check the sum.

\textbf{Python Code:} We will write a Python script to find the smallest number whose factors sum to 2022. We'll start with a function to calculate the sum of the factors and then iterate over potential candidates. Let's implement the solution in Python.
\begin{Verbatim}[fontsize=\footnotesize, breaklines=true, breakanywhere=true]
import sympy

def sum_of_factors(n):
    # Function to calculate sum of all factors 
    # of n
    factors = set(sympy.divisors(n))
    return sum(factors)

# We need to find the smallest n such that
# sum_of_factors(n) == 2022
target_sum = 2022
n = 1

while True:
    if sum_of_factors(n) == target_sum:
        smallest_good_number = n
        break
    n += 1

print(smallest_good_number)
\end{Verbatim}
\texttt{Output: 1342}

\smallskip
The smallest positive integer $n$ that has exactly three distinct positive factors whose sum equals 2022 is $\boxed{1342}$.
\end{tcolorbox}

\begin{tcolorbox}[breakable, enhanced, colback=white, colframe=black, boxrule=0.4pt, width=\columnwidth, enlarge left by=0mm, enlarge right by=0mm, left=2pt, right=2pt, fontupper=\small]
\textbf{Response B2} \hfill \textit{Output: 1342}

\smallskip
To solve the problem of finding the smallest positive integer that has three different positive divisors whose sum equals 2022, we need to delve into the properties of divisors. Suppose $n = p_1^{e_1} \cdot p_2^{e_2} \cdots p_k^{e_k}$ is the prime factorization of $n$. The total number of divisors of $n$ is given by $(e_1 + 1)(e_2 + 1) \cdots (e_k + 1)$. If $n$ has exactly three positive divisors greater than 1 and these three divisors sum up to 2022, the structure of $n$ needs to be such that it has divisors of specific forms.

\textbf{Analysis}

\textbf{1.} The number $n$ cannot be a perfect power of more than one prime because it would have more than three divisors. Hence, we must look for specific structures.

\textbf{2.} A number $n$ that is a perfect square of a prime, i.e., $n = p^2$, will have exactly three divisors: 1, $p$, and $p^2$. But since we require three different divisors greater than 1, we realize $p$ itself should be considered as having divisors $p$ and $p^2$ along with a new third different divisor (other than $p$ and $p^2$). So let's consider the structure that has more than just those simple prime power forms.

Since the simplest structure is a product of two different primes, let's start with factors like $n = p_1 \cdot p_2$. Such $n$ would have divisors 1, $p_1$, $p_2$, and $p_1 p_2$. But since we want exactly three different positive divisors (excluding 1) that sum to 2022, let's look into it systematically.

We will check numbers of the form $p^2 \cdot q$, since it would have divisors $1, p, q, p^2, pq, p^2q$. Here we are interested in checking $p^2, pq, p^2q$ or similar other combinations. However, another approach is checking sums systematically for candidate numbers directly.

We will write some code to check these systematically to find the smallest $n$ such that three of its distinct divisors (excluding 1) sum to 2022. The plan is to check each number's divisors and see if any three of them (excluding 1) sum up to 2022. We'll start from the smallest numbers.

We will write a function to get the divisors of a number and then check the sums.
\begin{Verbatim}[fontsize=\footnotesize, breaklines=true, breakanywhere=true,breaksymbolleft={},breaksymbolright={}]
def get_divisors(n):
    """
    Returns a sorted list of all divisors of n.
    """
    divisors = set()
    for i in range(1, int(n**0.5) + 1):
        if n % i == 0:
            divisors.add(i)
            divisors.add(n // i)
    return sorted(divisors)

def has_three_divisors_summing_to_target(n, \
                                    target):
    """
    Checks if there are three distinct divisors 
    (excluding 1) that sum to target.
    """
    divisors = get_divisors(n)
    if 1 in divisors:
        divisors.remove(1)
    for i in range(len(divisors)):
        for j in range(i + 1, len(divisors)):
            for k in range(j + 1, \
                            len(divisors)):
                if divisors[i] + divisors[j] \
                    + divisors[k] == target:
                    return True
    return False

def find_smallest_good_number(target):
    n = 2
    while True:
        if has_three_divisors_summing_to_target(n, target):
            return n
        n += 1

# We are looking for the smallest number such 
# that 3 different divisors (excluding 1) sum 
# to 2022
smallest_good_number = \
                find_smallest_good_number(2022)
print(smallest_good_number)
\end{Verbatim}
\texttt{Output: 1342}

\smallskip
The smallest positive integer $n$ such that three of its distinct positive divisors (excluding 1) sum to 2022 is $\boxed{1342}$.
\end{tcolorbox}

\begin{tcolorbox}[breakable, enhanced, colback=white, colframe=black, boxrule=0.4pt, width=\columnwidth, enlarge left by=0mm, enlarge right by=0mm, left=2pt, right=2pt, fontupper=\small]
\textbf{Response B3} \hfill \textit{Output: 1349}

\smallskip
To solve this problem, we need to find the smallest positive integer $n$ that has three distinct divisors $d_1, d_2$, and $d_3$ such that their sum is 2022. Let's break down the problem step by step.

\textbf{1. Identify the Structure of Divisors:}
\begin{itemize}
    \item For a number $n$ to have three distinct divisors, they must be in the form $1, d, \frac{n}{d}$ where $d$ is a divisor of $n$ other than 1 and $n$. The sum of these three divisors is $1 + d + \frac{n}{d} = 2022$.
    \item This equation can be rearranged to: $1 + d + \frac{n}{d} = 2022$ or $d + \frac{n}{d} = 2021$. Thus, $n = d(2021 - d)$.
\end{itemize}

\textbf{2. Find the Smallest $n$:}
\begin{itemize}
    \item We need to find pairs $(d, \frac{n}{d})$ such that their sum is 2021 and $d \neq 1, \frac{n}{d}$. Both $d$ and $\frac{n}{d}$ must be integers, so $d$ must be a divisor of $n$ and $2021 - d$ must be an integer greater than $d$ (because $d \neq \frac{n}{d}$).
    \item The smallest possible $d$ greater than 1 is 2, but let's check all possible values systematically.
\end{itemize}

We'll iterate over possible divisors $d$ that are less than $\sqrt{2021} \approx 44.95$. Any divisor $d$ must satisfy $d < 1011$ since if $d \geq 1011$, $2021 - d \leq 1010$, and their product would not be $n$. Let's write a Python code to iterate and find the smallest $n$.

We'll check each integer $d$ from 2 up to 1010 (since $d \neq 1$ and $d$ cannot be greater than or equal to 1011) to see if $2021 - d$ is an integer and if it's greater than $d$. For each valid $d$, we'll compute $n = d \times (2021 - d)$ and check if the sum of the divisors $1 + d + \frac{n}{d} = 2022$. We will then keep track of the smallest $n$ found.

Let's write the Python code to implement this approach:
\begin{Verbatim}[fontsize=\footnotesize, breaklines=true, breakanywhere=true]
def smallest_good_number(target_sum):
    # Using a large initial value
    smallest_n = float('inf')  
    
    # We start from 2 since 1 is the smallest 
    # possible divisor
    for d in range(2, target_sum // 2):  
        if (target_sum - d) > d and d * \
                        (target_sum - d) > 0:
            n = d * (target_sum - d)
            if len(divisors(n)) == 3 and \
                sum(divisors(n)) == target_sum:
                smallest_n = min(smallest_n, n)

    return smallest_n

def divisors(n):
    divs = []
    for i in range(1, int(n**0.5) + 1):
        if n % i == 0:
            divs.append(i)
            if i != n // i:
                divs.append(n // i)
    return sorted(divs)

target_sum = 2022
smallest_n = smallest_good_number(target_sum)
print(smallest_n)
\end{Verbatim}
\texttt{Output: 1349}

\smallskip
The smallest positive integer that has three distinct positive divisors whose sum is 2022 is $\boxed{1349}$.

\smallskip
\end{tcolorbox}

\newpage
\section{Prompts}
\label{app:prompt}


\begin{tcolorbox}[breakable, enhanced, colback=white, colframe=black, boxrule=0.4pt, width=\columnwidth, enlarge left by=0mm, enlarge right by=0mm, left=2pt, right=2pt, fontupper=\small, title=\textbf{Prompt used in Figure~\ref{fig:problem}, Within-Step Diversity}, colbacktitle=gray!20, coltitle=black]

\textbf{System Prompt:}\\
You are an expert evaluator for AI model outputs. Your task is to assess, for each group of answers independently, how diverse or varied the answers WITHIN that group are. Focus on whether answers in the same group share the same reasoning approach, arrive at the same conclusion, and have similar quality.

\smallskip
\textbf{User Prompt:}
\begin{Verbatim}[fontsize=\footnotesize, breaklines=true, breakanywhere=true,breaksymbolleft={}, breaksymbolright={}]
## Question
{q_trunc}

---

## Group A — {len(orig_answers)} answer(s)
{fmt_answers(orig_answers)}

---

## Group B — {len(filt_answers)} answer(s)
{fmt_answers(filt_answers)}

---

## Task
For **each group independently**, rate the INTERNAL DIVERSITY of the answers within that group on a scale from 1 to 5:

1 = All answers are nearly identical — same reasoning path, same final answer, just minor wording differences
2 = Mostly similar — same general approach but slight variations in steps or details
3 = Moderately diverse — some answers take a noticeably different approach or differ in quality
4 = Highly diverse — answers use clearly different methods, or some are correct while others are wrong
5 = Extremely diverse — answers are completely different in approach, quality, or conclusion

Consider for each group:
- Do all answers reach the same final answer?
- Do they all use the same mathematical method or strategy?
- Is there variation in correctness or reasoning quality?

Respond with ONLY the following two lines (no explanation):
Group A Score: <integer from 1 to 5>
Group B Score: <integer from 1 to 5>
\end{Verbatim}
\end{tcolorbox}

\begin{tcolorbox}[breakable, enhanced, colback=white, colframe=black, boxrule=0.4pt, width=\columnwidth, enlarge left by=0mm, enlarge right by=0mm, left=2pt, right=2pt, fontupper=\small, title=\textbf{Prompt used in Figure~\ref{fig:problem}, Across-Step Diversity}, colbacktitle=gray!20, coltitle=black]

\smallskip
\textbf{System Prompt:} \\ 
You are comparing two groups of answers generated by the same model at two different training stages for the same math question.

\smallskip
\textbf{User Prompt:}

\begin{Verbatim}[fontsize=\footnotesize, breaklines=true, breakanywhere=true, breaksymbolleft={}, breaksymbolright={}]
## Question
{q}
---
## Group A (earlier training stage) — {len(answers_step1)} answer(s)
{fmt(answers_step1)}
---
## Group B (later training stage) — {len(answers_step2)} answer(s)
{fmt(answers_step2)}
---
## Task
Compare Group B against Group A and judge whether Group B contains any NEW and DIFFERENT reasoning paths.
Definition:
- A "reasoning path" means the core line of reasoning, solution strategy, or analytical framework used to answer the question.
- Count a path as "new" only if Group B uses a substantively different reasoning idea that does not appear in Group A.
- Differences in wording, order, formatting, or level of detail do NOT count as a new reasoning path if the core reasoning is the same.
Scoring rule:
- Score 1: Group B has 0 new reasoning paths compared with Group A
- Score 2: Group B has 1 new reasoning path compared with Group A
- Score 3: Group B has 2 or more new reasoning paths compared with Group A
Important:
- Focus only on whether Group B introduces new reasoning paths relative to Group A.
- Do not score based on correctness, quality, or final answer unless they reflect a genuinely new reasoning path.
- Each group has at most 3 answers, so cap the score at 3.
Respond with ONLY the following line (no explanation):
Score: <integer from 1 to 3>
\end{Verbatim}
\end{tcolorbox}

\begin{tcolorbox}[breakable, enhanced, colback=white, colframe=black, boxrule=0.4pt, width=\columnwidth, enlarge left by=0mm, enlarge right by=0mm, left=2pt, right=2pt, fontupper=\small, title=\textbf{Prompt used in~\ref{subsec: logit reflect}}, colbacktitle=gray!20, coltitle=black]

\smallskip
\textbf{User Prompt:}

\begin{Verbatim}[fontsize=\footnotesize, breaklines=true, breakanywhere=true,breaksymbolleft={},
breaksymbolright={}]
The following answer is wrong. Analyze the reasoning errors and categorize it into the following categories:

<error_category_list>
  Completely irrelevant response
  Misinterpretation of problem setup
  Incorrect mathematical modeling
  Unverified assumptions
  Probability and statistics error
  Incorrect use of mathematical theorems / formulas / principles
  Errors in basic calculations
  Miscounting error
  Missing cases or edge cases
  Implementation or algorithmic error
  Self-contradiction or abandonment of correct intermediate result
</error_category_list>

<question> {question} </question>
<ground_truth> {ground_truth} </ground_truth>
<answer> {answer} </answer>

Output format:
<result>
  <mistake_analysis>
    Step-by-step analysis of the reasoning process and where it goes wrong.
  </mistake_analysis>
  <error_category>
    One of the categories in the error category list.
  </error_category>
</result>
</think>
\end{Verbatim}
\end{tcolorbox}
\clearpage
\section{Top-1 Eigen Ratio}
\label{app:top1_eigen_ratio}

We provide a formal definition of the \textit{Top-1 Eigen Ratio}. Let $\mathbf{\Sigma} \in \mathbb{R}^{d \times d}$ denote the covariance matrix of the stored logits, and let its eigenvalues be ordered as $\lambda_1 \ge \lambda_2 \ge \dots \ge \lambda_d \ge 0$. The \textit{Top-1 Eigen Ratio} is defined as
\[
\mathrm{Top\mbox{-}1\ Eigen\ Ratio} = \frac{\lambda_1}{\sum_{j=1}^{d} \lambda_j}.
\]

Intuitively, $\lambda_1$ measures how much variance in logit space is explained by the dominant principal direction, while $\sum_{j=1}^{d} \lambda_j$ corresponds to the total variance. 

%% file: paper.bbl
\begin{thebibliography}{45}
\providecommand{\natexlab}[1]{#1}
\providecommand{\url}[1]{\texttt{#1}}
\expandafter\ifx\csname urlstyle\endcsname\relax
  \providecommand{\doi}[1]{doi: #1}\else
  \providecommand{\doi}{doi: \begingroup \urlstyle{rm}\Url}\fi

\bibitem[Shao et~al.(2024)Shao, Wang, Zhu, Xu, Song, Zhang, Li, Wu, and Guo]{grpo}
Zhihong Shao, Peiyi Wang, Qihao Zhu, Runxin Xu, Junxiao Song, Mingchuan Zhang, Y.~K. Li, Y.~Wu, and Daya Guo.
\newblock Deepseekmath: Pushing the limits of mathematical reasoning in open language models.
\newblock \emph{CoRR}, abs/2402.03300, 2024.
\newblock \doi{10.48550/ARXIV.2402.03300}.
\newblock URL \url{https://doi.org/10.48550/arXiv.2402.03300}.

\bibitem[Dou et~al.(2024)Dou, Liu, Jia, Xiong, Zhou, Shen, Shan, Huang, Wang, Fan, Xi, Zhou, Ji, Zheng, Zhang, Huang, and Gui]{StepCoder}
Shihan Dou, Yan Liu, Haoxiang Jia, Limao Xiong, Enyu Zhou, Wei Shen, Junjie Shan, Caishuang Huang, Xiao Wang, Xiaoran Fan, Zhiheng Xi, Yuhao Zhou, Tao Ji, Rui Zheng, Qi~Zhang, Xuanjing Huang, and Tao Gui.
\newblock Stepcoder: Improve code generation with reinforcement learning from compiler feedback.
\newblock \emph{CoRR}, abs/2402.01391, 2024.
\newblock \doi{10.48550/ARXIV.2402.01391}.
\newblock URL \url{https://doi.org/10.48550/arXiv.2402.01391}.

\bibitem[Shojaee et~al.(2023)Shojaee, Jain, Tipirneni, and Reddy]{ExecutionCodeGeneration}
Parshin Shojaee, Aneesh Jain, Sindhu Tipirneni, and Chandan~K. Reddy.
\newblock Execution-based code generation using deep reinforcement learning.
\newblock \emph{Trans. Mach. Learn. Res.}, 2023, 2023.
\newblock URL \url{https://openreview.net/forum?id=0XBuaxqEcG}.

\bibitem[Ouyang et~al.(2022)Ouyang, Wu, Jiang, Almeida, Wainwright, Mishkin, Zhang, Agarwal, Slama, Ray, Schulman, Hilton, Kelton, Miller, Simens, Askell, Welinder, Christiano, Leike, and Lowe]{InstructGPT}
Long Ouyang, Jeffrey Wu, Xu~Jiang, Diogo Almeida, Carroll~L. Wainwright, Pamela Mishkin, Chong Zhang, Sandhini Agarwal, Katarina Slama, Alex Ray, John Schulman, Jacob Hilton, Fraser Kelton, Luke Miller, Maddie Simens, Amanda Askell, Peter Welinder, Paul~F. Christiano, Jan Leike, and Ryan Lowe.
\newblock Training language models to follow instructions with human feedback.
\newblock In Sanmi Koyejo, S.~Mohamed, A.~Agarwal, Danielle Belgrave, K.~Cho, and A.~Oh, editors, \emph{Advances in Neural Information Processing Systems 35: Annual Conference on Neural Information Processing Systems 2022, NeurIPS 2022, New Orleans, LA, USA, November 28 - December 9, 2022}, 2022.
\newblock URL \url{http://papers.nips.cc/paper\_files/paper/2022/hash/b1efde53be364a73914f58805a001731-Abstract-Conference.html}.

\bibitem[Swamy et~al.(2025)Swamy, Choudhury, Sun, Wu, and Bagnell]{allroads2likelihood}
Gokul Swamy, Sanjiban Choudhury, Wen Sun, Zhiwei~Steven Wu, and J.~Andrew Bagnell.
\newblock All roads lead to likelihood: The value of reinforcement learning in fine-tuning.
\newblock \emph{CoRR}, abs/2503.01067, 2025.
\newblock \doi{10.48550/ARXIV.2503.01067}.
\newblock URL \url{https://doi.org/10.48550/arXiv.2503.01067}.

\bibitem[Yu et~al.(2025)Yu, Zhang, Zhu, Yuan, Zuo, Yue, Fan, Liu, Liu, Liu, Lin, Lin, Ma, Sheng, Tong, Zhang, Zhang, Zhang, Zhu, Zhu, Chen, Chen, Wang, Yu, Dai, Song, Wei, Zhou, Liu, Ma, Zhang, Yan, Qiao, Wu, and Wang]{DAPO}
Qiying Yu, Zheng Zhang, Ruofei Zhu, Yufeng Yuan, Xiaochen Zuo, Yu~Yue, Tiantian Fan, Gaohong Liu, Lingjun Liu, Xin Liu, Haibin Lin, Zhiqi Lin, Bole Ma, Guangming Sheng, Yuxuan Tong, Chi Zhang, Mofan Zhang, Wang Zhang, Hang Zhu, Jinhua Zhu, Jiaze Chen, Jiangjie Chen, Chengyi Wang, Hongli Yu, Weinan Dai, Yuxuan Song, Xiangpeng Wei, Hao Zhou, Jingjing Liu, Wei{-}Ying Ma, Ya{-}Qin Zhang, Lin Yan, Mu~Qiao, Yonghui Wu, and Mingxuan Wang.
\newblock {DAPO:} an open-source {LLM} reinforcement learning system at scale.
\newblock \emph{CoRR}, abs/2503.14476, 2025.
\newblock \doi{10.48550/ARXIV.2503.14476}.
\newblock URL \url{https://doi.org/10.48550/arXiv.2503.14476}.

\bibitem[Sinha et~al.(2026)Sinha, Elango, and Liu]{ips}
Abhijeet Sinha, Sundari Elango, and Dianbo Liu.
\newblock Expected return causes outcome-level mode collapse in reinforcement learning and how to fix it with inverse probability scaling.
\newblock \emph{CoRR}, abs/2601.21669, 2026.
\newblock \doi{10.48550/ARXIV.2601.21669}.
\newblock URL \url{https://doi.org/10.48550/arXiv.2601.21669}.

\bibitem[Chen et~al.(2025{\natexlab{a}})Chen, Han, ands Bo~Han, Bai, Schutze, and Wong]{EEPO}
Liang Chen, Xueting Han, Qizhou~Wang ands Bo~Han, Jing Bai, Hinrich Schutze, and Kam{-}Fai Wong.
\newblock {EEPO:} exploration-enhanced policy optimization via sample-then-forget.
\newblock \emph{CoRR}, abs/2510.05837, 2025{\natexlab{a}}.
\newblock \doi{10.48550/ARXIV.2510.05837}.
\newblock URL \url{https://doi.org/10.48550/arXiv.2510.05837}.

\bibitem[Zhu et~al.(2025{\natexlab{a}})Zhu, Xia, Wei, Chen, Chen, and Meng]{negative_cdq}
Xinyu Zhu, Mengzhou Xia, Zhepei Wei, Wei{-}Lin Chen, Danqi Chen, and Yu~Meng.
\newblock The surprising effectiveness of negative reinforcement in {LLM} reasoning.
\newblock \emph{CoRR}, abs/2506.01347, 2025{\natexlab{a}}.
\newblock \doi{10.48550/ARXIV.2506.01347}.
\newblock URL \url{https://doi.org/10.48550/arXiv.2506.01347}.

\bibitem[Haarnoja et~al.(2018)Haarnoja, Zhou, Abbeel, and Levine]{HaarnojaZAL18}
Tuomas Haarnoja, Aurick Zhou, Pieter Abbeel, and Sergey Levine.
\newblock Soft actor-critic: Off-policy maximum entropy deep reinforcement learning with a stochastic actor.
\newblock In Jennifer~G. Dy and Andreas Krause, editors, \emph{Proceedings of the 35th International Conference on Machine Learning, {ICML} 2018, Stockholmsm{\"{a}}ssan, Stockholm, Sweden, July 10-15, 2018}, volume~80 of \emph{Proceedings of Machine Learning Research}, pages 1856--1865. {PMLR}, 2018.
\newblock URL \url{http://proceedings.mlr.press/v80/haarnoja18b.html}.

\bibitem[Wan et~al.(2026)Wan, Shen, Dou, Zhou, Zhang, Wang, Shen, Xiong, Tao, Zhong, et~al.]{wan2026dsdr}
Zhongwei Wan, Yun Shen, Zhihao Dou, Donghao Zhou, Yu~Zhang, Xin Wang, Hui Shen, Jing Xiong, Chaofan Tao, Zixuan Zhong, et~al.
\newblock Dsdr: Dual-scale diversity regularization for exploration in llm reasoning.
\newblock \emph{arXiv preprint arXiv:2602.19895}, 2026.

\bibitem[Mnih et~al.(2016)Mnih, Badia, Mirza, Graves, Lillicrap, Harley, Silver, and Kavukcuoglu]{entropy-2016}
Volodymyr Mnih, Adri{\`{a}}~Puigdom{\`{e}}nech Badia, Mehdi Mirza, Alex Graves, Timothy~P. Lillicrap, Tim Harley, David Silver, and Koray Kavukcuoglu.
\newblock Asynchronous methods for deep reinforcement learning.
\newblock In Maria{-}Florina Balcan and Kilian~Q. Weinberger, editors, \emph{Proceedings of the 33nd International Conference on Machine Learning, {ICML} 2016, New York City, NY, USA, June 19-24, 2016}, volume~48 of \emph{{JMLR} Workshop and Conference Proceedings}, pages 1928--1937. JMLR.org, 2016.
\newblock URL \url{http://proceedings.mlr.press/v48/mniha16.html}.

\bibitem[Holroyd and Coles(2002)]{neural_error}
Clay Holroyd and Michael Coles.
\newblock The neural basis of human error processing: Reinforcement learning, dopamine, and the error-related negativity.
\newblock \emph{Psychological Review}, 109:\penalty0 679--709, 11 2002.
\newblock \doi{10.1037/0033-295X.109.4.679}.

\bibitem[McInnes et~al.(2017)McInnes, Healy, and Astels]{hdbscan}
Leland McInnes, John Healy, and Steve Astels.
\newblock hdbscan: Hierarchical density based clustering.
\newblock \emph{J. Open Source Softw.}, 2\penalty0 (11):\penalty0 205, 2017.
\newblock \doi{10.21105/JOSS.00205}.
\newblock URL \url{https://doi.org/10.21105/joss.00205}.

\bibitem[OpenAI(2024)]{mathvarify}
OpenAI.
\newblock Openai o1 system card.
\newblock \emph{CoRR}, abs/2412.16720, 2024.
\newblock \doi{10.48550/ARXIV.2412.16720}.
\newblock URL \url{https://doi.org/10.48550/arXiv.2412.16720}.

\bibitem[Zhu et~al.(2025{\natexlab{b}})Zhu, Huang, Lyu, Zhang, Li, Shi, Wu, Mu, Wang, Zhao, Jin, Cheng, Liang, Zhang, Zhang, Du, Guo, Hu, and Chen]{zhu2025codevr1}
Yaoyu Zhu, Di~Huang, Hanqi Lyu, Xiaoyun Zhang, Chongxiao Li, Wenxuan Shi, Yutong Wu, Jianan Mu, Jinghua Wang, Yang Zhao, Pengwei Jin, Shuyao Cheng, Shengwen Liang, Xishan Zhang, Rui Zhang, Zidong Du, Qi~Guo, Xing Hu, and Yunji Chen.
\newblock Codev-r1: Reasoning-enhanced verilog generation, 2025{\natexlab{b}}.
\newblock URL \url{https://arxiv.org/abs/2505.24183}.

\bibitem[Zhang et~al.(2025{\natexlab{a}})Zhang, Gu, Li, Ma, Bai, Zhang, Zhang, Zhou, He, and Tang]{ThinkingWithVideos}
Haoji Zhang, Xin Gu, Jiawen Li, Chixiang Ma, Sule Bai, Chubin Zhang, Bowen Zhang, Zhichao Zhou, Dongliang He, and Yansong Tang.
\newblock Thinking with videos: Multimodal tool-augmented reinforcement learning for long video reasoning.
\newblock \emph{CoRR}, abs/2508.04416, 2025{\natexlab{a}}.
\newblock \doi{10.48550/ARXIV.2508.04416}.
\newblock URL \url{https://doi.org/10.48550/arXiv.2508.04416}.

\bibitem[Chen et~al.(2025{\natexlab{b}})Chen, Shen, Huang, Zhou, Lin, Cai, Yu, Bu, Shi, and Qiao]{chen2025learning}
Yang Chen, Yufan Shen, Wenxuan Huang, Sheng Zhou, Qunshu Lin, Xinyu Cai, Zhi Yu, Jiajun Bu, Botian Shi, and Yu~Qiao.
\newblock Learning only with images: Visual reinforcement learning with reasoning, rendering, and visual feedback.
\newblock \emph{arXiv preprint arXiv:2507.20766}, 2025{\natexlab{b}}.

\bibitem[Zhang et~al.(2025{\natexlab{b}})Zhang, Wang, Qiu, Reddy, and Agrawal]{RERANK}
Le~Zhang, Bo~Wang, Xipeng Qiu, Siva Reddy, and Aishwarya Agrawal.
\newblock {REARANK:} reasoning re-ranking agent via reinforcement learning.
\newblock In Christos Christodoulopoulos, Tanmoy Chakraborty, Carolyn Rose, and Violet Peng, editors, \emph{Proceedings of the 2025 Conference on Empirical Methods in Natural Language Processing, {EMNLP} 2025, Suzhou, China, November 4-9, 2025}, pages 2458--2471. Association for Computational Linguistics, 2025{\natexlab{b}}.
\newblock \doi{10.18653/V1/2025.EMNLP-MAIN.125}.
\newblock URL \url{https://doi.org/10.18653/v1/2025.emnlp-main.125}.

\bibitem[Lightman et~al.(2024)Lightman, Kosaraju, Burda, Edwards, Baker, Lee, Leike, Schulman, Sutskever, and Cobbe]{LightmanKBEBLLS24}
Hunter Lightman, Vineet Kosaraju, Yuri Burda, Harrison Edwards, Bowen Baker, Teddy Lee, Jan Leike, John Schulman, Ilya Sutskever, and Karl Cobbe.
\newblock Let's verify step by step.
\newblock In \emph{The Twelfth International Conference on Learning Representations, {ICLR} 2024, Vienna, Austria, May 7-11, 2024}. OpenReview.net, 2024.
\newblock URL \url{https://openreview.net/forum?id=v8L0pN6EOi}.

\bibitem[Uesato et~al.(2022)Uesato, Kushman, Kumar, Song, Siegel, Wang, Creswell, Irving, and Higgins]{outcome-based-feedback}
Jonathan Uesato, Nate Kushman, Ramana Kumar, H.~Francis Song, Noah~Y. Siegel, Lisa Wang, Antonia Creswell, Geoffrey Irving, and Irina Higgins.
\newblock Solving math word problems with process- and outcome-based feedback.
\newblock \emph{CoRR}, abs/2211.14275, 2022.
\newblock \doi{10.48550/ARXIV.2211.14275}.
\newblock URL \url{https://doi.org/10.48550/arXiv.2211.14275}.

\bibitem[Schulman et~al.(2017)Schulman, Wolski, Dhariwal, Radford, and Klimov]{ppo}
John Schulman, Filip Wolski, Prafulla Dhariwal, Alec Radford, and Oleg Klimov.
\newblock Proximal policy optimization algorithms.
\newblock \emph{CoRR}, abs/1707.06347, 2017.
\newblock URL \url{http://arxiv.org/abs/1707.06347}.

\bibitem[Hu et~al.(2026)Hu, Wang, He, Wu, Zhao, Ng, Breazeal, Luu, Park, and Hooi]{Rewarding-the-Rare}
Zhiyuan Hu, Yucheng Wang, Yufei He, Jiaying Wu, Yilun Zhao, See{-}Kiong Ng, Cynthia Breazeal, Anh~Tuan Luu, Hae~Won Park, and Bryan Hooi.
\newblock Rewarding the rare: Uniqueness-aware {RL} for creative problem solving in llms.
\newblock \emph{CoRR}, abs/2601.08763, 2026.
\newblock \doi{10.48550/ARXIV.2601.08763}.
\newblock URL \url{https://doi.org/10.48550/arXiv.2601.08763}.

\bibitem[Song et~al.(2025)Song, Kempe, and Munos]{Outcome-based-Exploration}
Yuda Song, Julia Kempe, and R{\'{e}}mi Munos.
\newblock Outcome-based exploration for {LLM} reasoning.
\newblock \emph{CoRR}, abs/2509.06941, 2025.
\newblock \doi{10.48550/ARXIV.2509.06941}.
\newblock URL \url{https://doi.org/10.48550/arXiv.2509.06941}.

\bibitem[Li et~al.(2024{\natexlab{a}})Li, Yang, Wang, Zhu, Zhou, Qiao, Wang, Li, Lu, and Dai]{Auto-MC-Reward}
Hao Li, Xue Yang, Zhaokai Wang, Xizhou Zhu, Jie Zhou, Yu~Qiao, Xiaogang Wang, Hongsheng Li, Lewei Lu, and Jifeng Dai.
\newblock Auto mc-reward: Automated dense reward design with large language models for minecraft.
\newblock In \emph{{IEEE/CVF} Conference on Computer Vision and Pattern Recognition, {CVPR} 2024, Seattle, WA, USA, June 16-22, 2024}, pages 16426--16435. {IEEE}, 2024{\natexlab{a}}.
\newblock \doi{10.1109/CVPR52733.2024.01554}.
\newblock URL \url{https://doi.org/10.1109/CVPR52733.2024.01554}.

\bibitem[Yao et~al.(2025)Yao, Liu, Lin, Lu, Wang, and Zhang]{Multi-Objective-Evolution}
Shunyu Yao, Fei Liu, Xi~Lin, Zhichao Lu, Zhenkun Wang, and Qingfu Zhang.
\newblock Multi-objective evolution of heuristic using large language model.
\newblock In Toby Walsh, Julie Shah, and Zico Kolter, editors, \emph{AAAI-25, Sponsored by the Association for the Advancement of Artificial Intelligence, February 25 - March 4, 2025, Philadelphia, PA, {USA}}, pages 27144--27152. {AAAI} Press, 2025.
\newblock \doi{10.1609/AAAI.V39I25.34922}.
\newblock URL \url{https://doi.org/10.1609/aaai.v39i25.34922}.

\bibitem[Qu et~al.(2025)Qu, Jiang, Wang, Mao, Wang, Liu, and Ji]{LatentReward}
Yun Qu, Yuhang Jiang, Boyuan Wang, Yixiu Mao, Cheems Wang, Chang Liu, and Xiangyang Ji.
\newblock Latent reward: Llm-empowered credit assignment in episodic reinforcement learning.
\newblock In Toby Walsh, Julie Shah, and Zico Kolter, editors, \emph{AAAI-25, Sponsored by the Association for the Advancement of Artificial Intelligence, February 25 - March 4, 2025, Philadelphia, PA, {USA}}, pages 20095--20103. {AAAI} Press, 2025.
\newblock \doi{10.1609/AAAI.V39I19.34213}.
\newblock URL \url{https://doi.org/10.1609/aaai.v39i19.34213}.

\bibitem[Hazra et~al.(2025)Hazra, Sygkounas, Persson, Loutfi, and Martires]{REvolve}
Rishi Hazra, Alkis Sygkounas, Andreas Persson, Amy Loutfi, and Pedro Zuidberg~Dos Martires.
\newblock Revolve: Reward evolution with large language models using human feedback.
\newblock In \emph{The Thirteenth International Conference on Learning Representations, {ICLR} 2025, Singapore, April 24-28, 2025}. OpenReview.net, 2025.
\newblock URL \url{https://openreview.net/forum?id=cJPUpL8mOw}.

\bibitem[Ma et~al.(2024)Ma, Liang, Wang, Huang, Bastani, Jayaraman, Zhu, Fan, and Anandkumar]{Eureka}
Yecheng~Jason Ma, William Liang, Guanzhi Wang, De{-}An Huang, Osbert Bastani, Dinesh Jayaraman, Yuke Zhu, Linxi Fan, and Anima Anandkumar.
\newblock Eureka: Human-level reward design via coding large language models.
\newblock In \emph{The Twelfth International Conference on Learning Representations, {ICLR} 2024, Vienna, Austria, May 7-11, 2024}. OpenReview.net, 2024.
\newblock URL \url{https://openreview.net/forum?id=IEduRUO55F}.

\bibitem[Meng et~al.(2022)Meng, Bau, Andonian, and Belinkov]{modeledit}
Kevin Meng, David Bau, Alex Andonian, and Yonatan Belinkov.
\newblock Locating and editing factual associations in {GPT}.
\newblock In Sanmi Koyejo, S.~Mohamed, A.~Agarwal, Danielle Belgrave, K.~Cho, and A.~Oh, editors, \emph{Advances in Neural Information Processing Systems 35: Annual Conference on Neural Information Processing Systems 2022, NeurIPS 2022, New Orleans, LA, USA, November 28 - December 9, 2022}, 2022.
\newblock URL \url{http://papers.nips.cc/paper\_files/paper/2022/hash/6f1d43d5a82a37e89b0665b33bf3a182-Abstract-Conference.html}.

\bibitem[Templeton et~al.(2024)Templeton, Conerly, Marcus, Lindsey, Bricken, Chen, Pearce, Citro, Ameisen, Jones, Cunningham, Turner, McDougall, MacDiarmid, Freeman, Sumers, Rees, Batson, Jermyn, Carter, Olah, and Henighan]{anthropicgoldengate}
Adly Templeton, Tom Conerly, Jonathan Marcus, Jack Lindsey, Trenton Bricken, Brian Chen, Adam Pearce, Craig Citro, Emmanuel Ameisen, Andy Jones, Hoagy Cunningham, Nicholas~L Turner, Callum McDougall, Monte MacDiarmid, C.~Daniel Freeman, Theodore~R. Sumers, Edward Rees, Joshua Batson, Adam Jermyn, Shan Carter, Chris Olah, and Tom Henighan.
\newblock Scaling monosemanticity: Extracting interpretable features from claude 3 sonnet.
\newblock \emph{Transformer Circuits Thread}, 2024.
\newblock URL \url{https://transformer-circuits.pub/2024/scaling-monosemanticity/index.html}.

\bibitem[Olsson et~al.(2022)Olsson, Elhage, Nanda, Joseph, DasSarma, Henighan, Mann, Askell, Bai, Chen, Conerly, Drain, Ganguli, Hatfield-Dodds, Hernandez, Johnston, Jones, Kernion, Lovitt, Ndousse, Amodei, Brown, Clark, Kaplan, McCandlish, and Olah]{anthropicinductionhead}
Catherine Olsson, Nelson Elhage, Neel Nanda, Nicholas Joseph, Nova DasSarma, Tom Henighan, Ben Mann, Amanda Askell, Yuntao Bai, Anna Chen, Tom Conerly, Dawn Drain, Deep Ganguli, Zac Hatfield-Dodds, Danny Hernandez, Scott Johnston, Andy Jones, Jackson Kernion, Liane Lovitt, Kamal Ndousse, Dario Amodei, Tom Brown, Jack Clark, Jared Kaplan, Sam McCandlish, and Chris Olah.
\newblock In-context learning and induction heads.
\newblock \emph{Transformer Circuits Thread}, 2022.
\newblock https://transformer-circuits.pub/2022/in-context-learning-and-induction-heads/index.html.

\bibitem[He et~al.(2025)He, Wang, Lin, Ge, Shu, Tang, Zhang, and Qiu]{lorsa}
Zhengfu He, Junxuan Wang, Rui Lin, Xuyang Ge, Wentao Shu, Qiong Tang, Junping Zhang, and Xipeng Qiu.
\newblock Towards understanding the nature of attention with low-rank sparse decomposition.
\newblock \emph{CoRR}, abs/2504.20938, 2025.
\newblock \doi{10.48550/ARXIV.2504.20938}.
\newblock URL \url{https://doi.org/10.48550/arXiv.2504.20938}.

\bibitem[He et~al.(2024{\natexlab{a}})He, Shu, Ge, Chen, Wang, Zhou, Liu, Guo, Huang, Wu, Jiang, and Qiu]{llamascope}
Zhengfu He, Wentao Shu, Xuyang Ge, Lingjie Chen, Junxuan Wang, Yunhua Zhou, Frances Liu, Qipeng Guo, Xuanjing Huang, Zuxuan Wu, Yu{-}Gang Jiang, and Xipeng Qiu.
\newblock Llama scope: Extracting millions of features from llama-3.1-8b with sparse autoencoders.
\newblock \emph{CoRR}, abs/2410.20526, 2024{\natexlab{a}}.
\newblock \doi{10.48550/ARXIV.2410.20526}.
\newblock URL \url{https://doi.org/10.48550/arXiv.2410.20526}.

\bibitem[Zhao et~al.(2025)Zhao, Koishekenov, Yang, Murray, and Cancedda]{VerifyCOT}
Zheng Zhao, Yeskendir Koishekenov, Xianjun Yang, Naila Murray, and Nicola Cancedda.
\newblock Verifying chain-of-thought reasoning via its computational graph.
\newblock \emph{CoRR}, abs/2510.09312, 2025.
\newblock \doi{10.48550/ARXIV.2510.09312}.
\newblock URL \url{https://doi.org/10.48550/arXiv.2510.09312}.

\bibitem[Tan et~al.(2025)Tan, Wang, He, Liao, Zhao, Lu, Liang, Zhao, and Liu]{BPO}
Yuqiao Tan, Minzheng Wang, Shizhu He, Huanxuan Liao, Chengfeng Zhao, Qiunan Lu, Tian Liang, Jun Zhao, and Kang Liu.
\newblock Bottom-up policy optimization: Your language model policy secretly contains internal policies.
\newblock \emph{CoRR}, abs/2512.19673, 2025.
\newblock \doi{10.48550/ARXIV.2512.19673}.
\newblock URL \url{https://doi.org/10.48550/arXiv.2512.19673}.

\bibitem[Cheng et~al.(2026)Cheng, Zeng, Dai, Chen, Wang, Xie, Huang, Yu, Hao, Li, Zhang, Zhang, Zhao, and Liang]{Engram}
Xin Cheng, Wangding Zeng, Damai Dai, Qinyu Chen, Bingxuan Wang, Zhenda Xie, Kezhao Huang, Xingkai Yu, Zhewen Hao, Yukun Li, Han Zhang, Huishuai Zhang, Dongyan Zhao, and Wenfeng Liang.
\newblock Conditional memory via scalable lookup: {A} new axis of sparsity for large language models.
\newblock \emph{CoRR}, abs/2601.07372, 2026.
\newblock \doi{10.48550/ARXIV.2601.07372}.
\newblock URL \url{https://doi.org/10.48550/arXiv.2601.07372}.

\bibitem[Hendrycks et~al.(2021)Hendrycks, Burns, Kadavath, Arora, Basart, Tang, Song, and Steinhardt]{HendrycksBKABTS21}
Dan Hendrycks, Collin Burns, Saurav Kadavath, Akul Arora, Steven Basart, Eric Tang, Dawn Song, and Jacob Steinhardt.
\newblock Measuring mathematical problem solving with the {MATH} dataset.
\newblock In Joaquin Vanschoren and Sai{-}Kit Yeung, editors, \emph{Proceedings of the Neural Information Processing Systems Track on Datasets and Benchmarks 1, NeurIPS Datasets and Benchmarks 2021, December 2021, virtual}, 2021.
\newblock URL \url{https://datasets-benchmarks-proceedings.neurips.cc/paper/2021/hash/be83ab3ecd0db773eb2dc1b0a17836a1-Abstract-round2.html}.

\bibitem[Liu et~al.(2025)Liu, Chen, Li, Qi, Pang, Du, Lee, and Lin]{R1-Zero-Like}
Zichen Liu, Changyu Chen, Wenjun Li, Penghui Qi, Tianyu Pang, Chao Du, Wee~Sun Lee, and Min Lin.
\newblock Understanding r1-zero-like training: {A} critical perspective.
\newblock \emph{CoRR}, abs/2503.20783, 2025.
\newblock \doi{10.48550/ARXIV.2503.20783}.
\newblock URL \url{https://doi.org/10.48550/arXiv.2503.20783}.

\bibitem[Sheng et~al.(2025)Sheng, Zhang, Ye, Wu, Zhang, Zhang, Peng, Lin, and Wu]{ShengZYWZZPL025}
Guangming Sheng, Chi Zhang, Zilingfeng Ye, Xibin Wu, Wang Zhang, Ru~Zhang, Yanghua Peng, Haibin Lin, and Chuan Wu.
\newblock Hybridflow: {A} flexible and efficient {RLHF} framework.
\newblock In \emph{Proceedings of the Twentieth European Conference on Computer Systems, EuroSys 2025, Rotterdam, The Netherlands, 30 March 2025 - 3 April 2025}, pages 1279--1297. {ACM}, 2025.
\newblock \doi{10.1145/3689031.3696075}.
\newblock URL \url{https://doi.org/10.1145/3689031.3696075}.

\bibitem[Cheng et~al.(2025)Cheng, Huang, Zhu, Dai, Zhao, Zhang, and Wei]{Cheng2025ReasoningWE}
Daixuan Cheng, Shaohan Huang, Xuekai Zhu, Bo~Dai, Wayne~Xin Zhao, Zhenliang Zhang, and Furu Wei.
\newblock Reasoning with exploration: An entropy perspective.
\newblock \emph{ArXiv}, abs/2506.14758, 2025.
\newblock URL \url{https://api.semanticscholar.org/CorpusID:279410424}.

\bibitem[Li et~al.(2024{\natexlab{b}})Li, Beeching, Tunstall, Lipkin, Soletskyi, Huang, Rasul, Yu, Jiang, Shen, et~al.]{li2024numinamath}
Jia Li, Edward Beeching, Lewis Tunstall, Ben Lipkin, Roman Soletskyi, Shengyi Huang, Kashif Rasul, Longhui Yu, Albert~Q Jiang, Ziju Shen, et~al.
\newblock Numinamath: The largest public dataset in ai4maths with 860k pairs of competition math problems and solutions.
\newblock \emph{Hugging Face repository}, 13\penalty0 (9):\penalty0 9, 2024{\natexlab{b}}.

\bibitem[Lewkowycz et~al.(2022)Lewkowycz, Andreassen, Dohan, Dyer, Michalewski, Ramasesh, Slone, Anil, Schlag, Gutman{-}Solo, Wu, Neyshabur, Gur{-}Ari, and Misra]{LewkowyczADDMRS22}
Aitor Lewkowycz, Anders Andreassen, David Dohan, Ethan Dyer, Henryk Michalewski, Vinay~V. Ramasesh, Ambrose Slone, Cem Anil, Imanol Schlag, Theo Gutman{-}Solo, Yuhuai Wu, Behnam Neyshabur, Guy Gur{-}Ari, and Vedant Misra.
\newblock Solving quantitative reasoning problems with language models.
\newblock In Sanmi Koyejo, S.~Mohamed, A.~Agarwal, Danielle Belgrave, K.~Cho, and A.~Oh, editors, \emph{Advances in Neural Information Processing Systems 35: Annual Conference on Neural Information Processing Systems 2022, NeurIPS 2022, New Orleans, LA, USA, November 28 - December 9, 2022}, 2022.
\newblock URL \url{http://papers.nips.cc/paper\_files/paper/2022/hash/18abbeef8cfe9203fdf9053c9c4fe191-Abstract-Conference.html}.

\bibitem[He et~al.(2024{\natexlab{b}})He, Luo, Bai, Hu, Thai, Shen, Hu, Han, Huang, Zhang, Liu, Qi, Liu, and Sun]{HeLBHTSHHHZLQL024}
Chaoqun He, Renjie Luo, Yuzhuo Bai, Shengding Hu, Zhen~Leng Thai, Junhao Shen, Jinyi Hu, Xu~Han, Yujie Huang, Yuxiang Zhang, Jie Liu, Lei Qi, Zhiyuan Liu, and Maosong Sun.
\newblock Olympiadbench: {A} challenging benchmark for promoting {AGI} with olympiad-level bilingual multimodal scientific problems.
\newblock In Lun{-}Wei Ku, Andre Martins, and Vivek Srikumar, editors, \emph{Proceedings of the 62nd Annual Meeting of the Association for Computational Linguistics (Volume 1: Long Papers), {ACL} 2024, Bangkok, Thailand, August 11-16, 2024}, pages 3828--3850. Association for Computational Linguistics, 2024{\natexlab{b}}.
\newblock \doi{10.18653/V1/2024.ACL-LONG.211}.
\newblock URL \url{https://doi.org/10.18653/v1/2024.acl-long.211}.

\bibitem[Rafailov et~al.(2023)Rafailov, Sharma, Mitchell, Manning, Ermon, and Finn]{DPO}
Rafael Rafailov, Archit Sharma, Eric Mitchell, Christopher~D. Manning, Stefano Ermon, and Chelsea Finn.
\newblock Direct preference optimization: Your language model is secretly a reward model.
\newblock In Alice Oh, Tristan Naumann, Amir Globerson, Kate Saenko, Moritz Hardt, and Sergey Levine, editors, \emph{Advances in Neural Information Processing Systems 36: Annual Conference on Neural Information Processing Systems 2023, NeurIPS 2023, New Orleans, LA, USA, December 10 - 16, 2023}, 2023.
\newblock URL \url{http://papers.nips.cc/paper\_files/paper/2023/hash/a85b405ed65c6477a4fe8302b5e06ce7-Abstract-Conference.html}.

\end{thebibliography}
